\PassOptionsToPackage{table}{xcolor}
\documentclass[11pt]{article}

\usepackage[final]{acl}

\usepackage{times}
\usepackage{latexsym}
\usepackage[T1]{fontenc}
\usepackage[utf8]{inputenc}
\usepackage{microtype}
\usepackage{inconsolata}
\usepackage{graphicx}
\graphicspath{{figures/}{latex/figures/}}
\usepackage{booktabs}
\usepackage{amsmath}
\usepackage{amssymb}
\usepackage{amsthm}
\usepackage{array}
\usepackage{tabularx}
\usepackage{xcolor}
\usepackage{fvextra}
\usepackage{tcolorbox}
\tcbuselibrary{breakable,skins}
\usepackage{tikz}
\usepackage{enumitem}
\usepackage{url}

\usetikzlibrary{arrows.meta,positioning,shapes.geometric}

\newtheorem{proposition}{Proposition}

\newcommand{\dataset}{GNRS-Bench}
\newcommand{\ours}{GNRS-Search}
\newcommand{\seval}{s_{\mathrm{eval}}}
\newcommand{\ssearch}{s_{\mathrm{search}}}
\newcommand{\hardok}{\textsc{hard-ok}}
\newcommand{\overall}{\textsc{Overall}}

\newcommand{\slot}[1]{\textsc{#1}}
\newcommand{\pct}[1]{#1\%}
\newcommand{\best}[1]{\textbf{#1}}

\newcolumntype{Y}{>{\raggedright\arraybackslash}X}
\newtcolorbox{promptbox}{%
  enhanced jigsaw,
  breakable,
  parbox=false,
  colback=gray!10,
  colframe=gray!10,
  boxrule=0pt,
  arc=0pt,
  outer arc=0pt,
  width=\linewidth,
  left=3pt,
  right=3pt,
  top=3pt,
  bottom=3pt,
  boxsep=0pt,
  before skip=4pt,
  after skip=8pt
}
\newcommand{\promptinput}[2]{%
\par\medskip\noindent\textbf{#1.}\nopagebreak\par\nobreak
\IfFileExists{latex/appendix_prompts/#2}{%
\begin{promptbox}
\VerbatimInput[fontsize=\footnotesize,breaklines=true,breakanywhere=true]{latex/appendix_prompts/#2}%
\end{promptbox}
}{%
\begin{promptbox}
\VerbatimInput[fontsize=\footnotesize,breaklines=true,breakanywhere=true]{appendix_prompts/#2}%
\end{promptbox}
}%
}

\title{Grounded Normative Rule Generation with Structured Search}

\author{%
  {\bfseries Fanqi Kong$^{1,2}$, Huaxiao Yin$^{3}$, Ruijie Zhang$^{3}$,} \\
  {\bfseries Xiaoyuan Zhang$^{1,2}$, Yizhe Huang$^{1,2}$, Jian Gao$^{2}$, Shuo Chen$^{2}$\thanks{Corresponding Author. \texttt{chenshuo@bigai.ai}, \texttt{s.c.zhu@pku.edu.cn}}, Song-Chun Zhu$^{1,2}$\textsuperscript{*}}\\ \\
  $^1$Peking University \quad \quad
  $^2$State Key Laboratory of General Artificial Intelligence, BIGAI \\
  $^3$University of Chinese Academy of Sciences
}

\begin{document}
\maketitle

\begin{abstract}
Normative rules like institutional charters and workplace policies must be both human-readable and operationally verifiable against actual environment records. However, current language generation and structured-output benchmarks primarily reward surface fluency or schema compliance, leaving operational grounding weakly tested. This creates a critical vulnerability where standard language models generate plausible-sounding policies that fail during enforcement because they rely on unavailable data logs or misaligned scopes. To address this challenge, we formalize the problem as Grounded Normative Rule Synthesis (GNRS) and introduce GNRS-Search, a framework that utilizes Markov Chain Monte Carlo (MCMC) sampling to optimize a discrete, five-slot And-Or Graph (AOG). By explicitly decoupling intermediate operational structure from final prose generation, this method isolates executable feasibility from writing style and allows rule failures to be localized prior to surface realization. We evaluate our approach on \dataset{}, a benchmark spanning 116 controlled goals across eight scene families, and RealCharter-Bench, which evaluates transfer to 53 real-world policy tasks with hidden source clauses. GNRS-Search raises average rubric quality from 68.8\% to 81.0\% and ranks first under a disclosed executable composite metric, while systematic slot interventions confirm that performance gains stem from robust operational logic rather than rhetorical tuning. Ultimately, by transforming automated rule drafting into an inspectable search problem, this work provides a foundational paradigm for deploying verifiable and compliance-ready personal agents within regulated environments.
\end{abstract}

\section{Introduction}
\label{sec:intro}

Normative rules, such as institutional charters and workplace policies, must fundamentally satisfy a dual requirement: they must be human-readable and operationally verifiable when compliance is audited~\citep{sergot1986british, solar2006combinatorial, governatori2010conceptually}. This dual purpose creates a basic tension for language generation. A generated clause may exhibit high surface fluency while entirely failing as a valid rule by relying on unavailable evidence or untraceable enforcement procedures~\citep{wiseman2017challenges,maynez2020faithfulness,ji2023survey}.

We formalize this problem as Grounded Normative Rule Synthesis, which we abbreviate as GNRS.  Given a normative goal, an institutional schema, and a record layer containing environment traces and logs, the task is to write a natural-language rule that is both readable and strictly backed by available data.  
Figure~\ref{fig:teaser} illustrates this setting with a restaurant attendance example. 
Direct prompting can produce a fluent policy that invokes unsupported evidence channels, such as tracking systems or staff testimony. The core difficulty goes beyond factual hallucination: rule writing requires coordinating scope, trigger, norm, evidence, and procedure. A small change in one component can invalidate the others; for example, a rule may require a procedure to verify a violation against an automated log, yet select verbal staff testimony as its evidence, which is not stored in any accessible record---the procedure then has no data source it can actually inspect. We therefore treat rule generation as a structured decision problem before surface realization, so that failures can be localized to the component that caused them.

While existing text-generation tools address related problems, they do not enforce operational grounding. 
Constrained decoding and structured-output methods improve syntactic validity by restricting generations to schemas or grammars \citep{scholak2021picard, willard2023efficient}, but they do not check whether the generated content is supported by available environment records. 
Iterative self-refinement can improve outputs through feedback loops \citep{madaan2023self, shinn2023reflexion}, but it lacks an external record layer for verifying evidence channels and procedures. 
Interactive language agents use environmental feedback to guide action selection \citep{yao2022react, shridhar2020alfworld}, whereas GNRS asks whether the generated rule text itself is backed by auditable data logs. Related multi-agent and social-interaction studies similarly highlight the role of environment structure in shaping agent behavior \citep{kong2024learning, huang2024adasociety, zhang2026world}. 
Our core objective is to identify structural configurations that make a rule operationally grounded before it is rendered as prose.

To solve this problem, we implement GNRS-Search.  Instead of generating the final text directly, our framework models the rule as a five-slot And-Or Graph~\citep{zhu2007stochastic, wu2011numerical}, which we call an AOG, representing intermediate operational commitments across scope, trigger, norm, evidence, and procedure.  The system builds a candidate space by combining deterministic pieces from rule templates, context schemas, and record logs with language model proposals.  We then use Markov Chain Monte Carlo sampling~\citep{hastings1970monte}, or MCMC, to iteratively edit this discrete structure.  The search explores candidate changes through four explicit operations including slot replacement, parameter edits, evidence swaps, and joint trigger-evidence adjustments.  An internal scoring function guides the search by evaluating record grounding and penalizing hard feasibility violations.  This approach successfully turns rule writing into an inspectable search problem where structural errors can be caught and localized before the final text is rendered.

We evaluate our method on the introduced \dataset{}, which covers 116 controlled goals across 8 scene families, and test transfer to real-derived policy tasks on RealCharter-Bench, which contains 53 tasks derived from authentic policies. Experimental results show that GNRS-Search improves average rubric quality from 68.8\% to 81.0\% on the main test. It also ranks first under our overall executability metric across different model backbones and out-of-distribution scenes. Human evaluations and sensitivity analyses confirm that these improvements come from correct operational structure rather than better phrasing.

This paper makes four contributions.
\begin{itemize}[leftmargin=*,noitemsep]
    \item We define GNRS as a structured-NLG task that tests whether a readable
    rule is also backed by available records.
    \item We introduce \dataset{} with 116 controlled goals and
    RealCharter-Bench with 53 real-derived policy tasks.
    \item We implement \ours{}, an AOG search method with MCMC proposals for
    selecting grounded rule structures before rendering.
    \item We show gains on main, robustness, and real-derived tests. Human
    calibration supports the rubric, while slot interventions show that the
    gains come from operational meaning.
\end{itemize}

\begin{figure*}[!htbp]
\centering
\includegraphics[width=\linewidth]{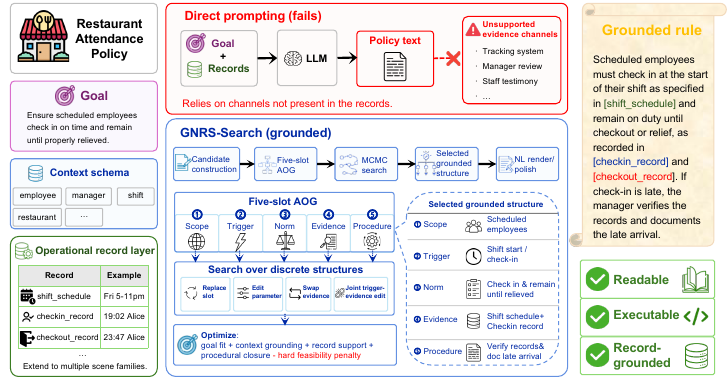}
\caption{Normative rules must be both readable and operationally grounded.
Direct prompting often satisfies the surface-language requirement but introduces
unsupported evidence channels or vague enforcement procedures. GNRS-Search
addresses this by searching over a five-slot operational structure, covering
scope, trigger, norm body, evidence, and procedure, before rendering the
selected structure into natural language grounded in concrete operational
records.}
\label{fig:teaser}
\end{figure*}

\section{Related Work}
\label{sec:related}
\paragraph{Grounded and constrained generation.}
Data-to-text and grounded NLG benchmarks study faithful realization from
structured inputs \citep{gardent2017webnlg,parikh2020totto,nan2021dart}.
Constrained decoding, JSON-mode generation, and grammar-guided decoding improve
parseability \citep{openai2024structured,dong2025xgrammar,tam2024let}, and recent work
further trains structured-output ability at the schema level
\citep{lu2025learning}. Best-of-\(N\) selection and verifier-style
reranking are also common prompting enhancements \citep{cobbe2021training}.
GNRS differs because the output is a normative clause whose validity depends not
only on format or fluency, but on whether its roles, triggers, evidence, and
procedures are grounded in operational records.

\paragraph{Rule grounding, agents, and normative NLP.}
Rule-oriented benchmarks test rule following or evidence-grounded reasoning
\citep{zhou2025rulearena,jiayang2025integround}, while GNRS asks
the model to synthesize the rule itself. Interactive agent work uses
environmental feedback to guide behavior
\citep{park2023generative,wang2023voyager,liu2024agentbench}; broader multi-agent benchmarks study error attribution in such systems \citep{kong2026aegis} and cross-environment learning \citep{zhang2025autoenv}, while social world models and social-interaction reasoning examine how physical and social dynamics are unified \citep{zhang2026social, kong2026siv}; GNRS instead
uses environmental structure as typed primitives and evidence channels that
constrain text generation. Legal and normative NLP has studied parsing,
retrieval, classification, and compliance reasoning
\citep{hendrycks2021cuad,guha2023legalbench,ren2024emergence,
holzenberger2020dataset,boella2006introduction}. GNRS is complementary, focusing on generating
enforceable normative language from goals and operational records. Adjacent LLM-based social agents are trained with adversarial objectives to improve interaction quality \citep{kong2025enhancing}.

\paragraph{Symbolic synthesis and constraints.}
Program synthesis and constraint solving show how discrete structures can be
searched under formal specifications \citep{solar2006combinatorial,
solar2008program}. GNRS-Search follows this spirit by
searching over a discrete rule structure before generating text. Unlike program
synthesis, however, the target is natural-language policy text, and the
constraints are partial grounding diagnostics over available records rather than
complete logical specifications.

\section{Preliminaries}
\label{sec:preliminaries}

We study normative text generation as a problem of producing language with an
operational interpretation. Let \(x=(g,C,D)\) denote a rule-writing context:
\(g\) is the normative intent, \(C\) is the institutional context schema
containing roles, locations, objects, and action vocabulary, and \(D\) is the
operational record layer containing event traces, logs, observations, and
evidence channels. The output is a natural-language rule \(y\). Unlike ordinary
open-ended text generation, the quality of \(y\) depends not only on fluency but
on whether it preserves the intent and can be operationalized against \(C\) and
\(D\).

\[
V(y;x)=(v_I,v_G,v_E,v_C,v_R).
\]

The components of \(V\) measure intent preservation, grounding, evidence use,
procedural closure, and readability. A scalar task score is any monotone
functional \(F(y;x)=u(V(y;x))\). We assume that a normative rule has a latent
operational structure \(a\) that mediates between intent and surface text. Let
\(K\) be the set of five slots: scope, trigger, norm, evidence, and procedure.
Then \(a=(a_k)_{k\in K}\). This structure is not the task output: the task
output remains natural language. It is instead a modeling assumption that lets
us reason about where grounding fails. A slot assignment is grounded when its
referenced actors, actions, evidence channels, and procedures can be interpreted
against \(C\) and \(D\).

Direct prompting estimates \(y\) in one step. Our view separates operational
interpretation from surface realization: construct a candidate space
\(\mathcal{A}(x)\), search for an operational structure \(a^*\), and then
realize it as text, where \(S\) is an internal score over structures and
\(R\) is a realization function:

\[
a^*=\arg\max_{a\in\mathcal{A}(x)} S(a;x),
\qquad
y^*=R(a^*,x).
\]
We keep \(S\) distinct from \(\seval\), the evaluation score over final language,
because a system should not win merely by optimizing the same metric used to
judge it.

The following proposition states the formal role of structure search. Let
\(G(a;x)\) be the task-validity score that a perfectly faithful realization of
structure \(a\) would obtain. Suppose \(S\) is calibrated to \(G\) within
\(\epsilon_C\), search returns an \(\epsilon_S\)-optimal structure under \(S\),
and realization loses at most \(\epsilon_R\) task validity.

\begin{proposition}
If \(|S(a;x)-G(a;x)|\leq \epsilon_C\) for all \(a\in\mathcal A(x)\),
\(S(\hat a;x)\geq \max_a S(a;x)-\epsilon_S\), and
\(F(R(\hat a,x);x)\geq G(\hat a;x)-\epsilon_R\), then
\[
F(R(\hat a,x);x)
\geq
\max_{a\in\mathcal A(x)}G(a;x)-\epsilon_S-2\epsilon_C-\epsilon_R .
\]
\end{proposition}

\begin{proof}
Let \(a_G\in\arg\max_a G(a;x)\). By calibration and search optimality,
\[
\begin{aligned}
G(\hat a;x)
&\geq S(\hat a;x)-\epsilon_C \\
&\geq S(a_G;x)-\epsilon_S-\epsilon_C \\
&\geq G(a_G;x)-\epsilon_S-2\epsilon_C .
\end{aligned}
\]
Combining this with the realization bound gives the result.
\end{proof}

\section{Evaluation Benchmarks}
\label{sec:benchmark}

We evaluate grounded rule synthesis using two complementary benchmarks. The primary testbed is \dataset{}, which provides a controlled simulation environment where available operational records are fully typed and systematically varied. To test transfer to real-derived policy tasks, we introduce RealCharter-Bench, a collection of evaluation scenarios derived directly from authentic institutional policies. Comprehensive domain distributions, record-layer vocabularies, and document source provenances are detailed in Appendix~\ref{app:benchmark-stats}.

\paragraph{\dataset{}.}
The benchmark evaluates whether a synthesis framework can write verifiable rules using only the information that an environment can actually log. Following the formalization in Section~\ref{sec:preliminaries}, each instance provides a target management goal alongside an institutional context schema and an operational record layer. For example, in an attendance scenario regulating restaurant shift coordination, the schema defines relevant roles and locations such as employees and service rooms, while the record layer exposes explicit data logs like check-in records. A valid synthesis method must anchor its generated scope, trigger, evidence, and procedure in these available logs instead of relying on unrecorded channels or subjective testimony.

The benchmark contains 116 goals distributed evenly across four functional categories including attendance, property boundaries, task execution, and mixed crisis objectives. We reserve 16 goals for development tuning and 88 goals for core in-distribution testing. To measure out-of-distribution performance, we hold out the entire \texttt{open\_space\_navigation} scene family as a separate 12-goal OOD split. Crucially, the underlying simulator distinguishes transient environmental observables from permanent evidence logs. This design ensures that generated rules can be evaluated after an interaction finishes by checking their compliance against durable data streams rather than fleeting environment states.

\paragraph{Evaluation protocol.}
The evaluation protocol decouples text generation from reference text imitation. During testing, a method receives the target management goal and the allowed context schema but is denied access to any human-authored golden rule. The framework must output a single natural-language rule clause whose semantic and programmatic commitments remain interpretable against the available environment logs. Final rule drafts are scored by a blinded multi-dimensional rubric to measure linguistic quality, while actual operational feasibility is verified via the binary \hardok{} gate over the underlying structural choices.

\paragraph{RealCharter-Bench.}
RealCharter-Bench provides an external check to ensure that synthesis frameworks generalize to authentic public policies. It comprises 53 tasks derived from 15 foundational open-source governance guidelines and corporate manuals such as the Python PSF and the GitLab Handbook. To prevent direct token copying, the original human-authored text is completely hidden from the generator. Instead, the dataset represents each task as a normalized policy card that exposes explicit fields for context, triggers, required conduct, and evidence requirements. This setup evaluates whether grounded structured search transfers effectively to real-derived policy tasks under a unified input-output abstraction; long-document understanding and policy extraction are outside the current scope.

\section{Method}
\label{sec:method}

The core design principle of \ours{} relies on separating surface prose generation from structural commitment synthesis. As shown in Figure~\ref{fig:pipeline}, instead of prompting a language model to output the final rule text \(y\) in an end-to-end pass, our framework first constructs grounded operational candidates for individual slots, executes a discrete optimization search over complete structural configurations \(a\), and realizes the optimal layout as natural language. This design enforces the scientific objective of anchoring free-text rule generation to an executable and auditable interpretation. We mathematically model the latent structure \(a\) as an And-Or Graph, abbreviated as AOG, which spans the five essential operational slots introduced in Section~\ref{sec:preliminaries} including scope, trigger, norm body, evidence policy, and procedure.

\begin{figure}[!htbp]
\centering
\resizebox{\linewidth}{!}{%
\begin{tikzpicture}[
  node distance=1.2cm,
  box/.style={draw, rounded corners, align=center, minimum width=2.6cm,
              minimum height=0.75cm, fill=gray!8},
  smallbox/.style={draw, rounded corners, align=center, minimum width=2.2cm,
              minimum height=0.65cm, fill=blue!5},
  arrow/.style={-{Latex[length=2mm]}, thick}
]
\node[box] (goal) {Atomic goal\\ \(g\)};
\node[box, right=of goal] (cand) {Candidate\\ generation};
\node[smallbox, above right=0.35cm and 0.55cm of cand] (rb) {rule-based};
\node[smallbox, below right=0.35cm and 0.55cm of cand] (llm) {context + records\\ + LLM};
\node[smallbox, right=of cand] (aog) {5-slot AOG\\ assembly};
\node[smallbox, right=of aog] (search) {MCMC search\\ with \(\ssearch\)};
\node[box, right=of search] (nl) {NL render\\ and polish};
\draw[arrow] (goal) -- (cand);
\draw[arrow] (cand) -- (aog);
\draw[arrow] (rb) -- (aog);
\draw[arrow] (llm) -- (aog);
\draw[arrow] (aog) -- (search);
\draw[arrow] (search) -- (nl);
\end{tikzpicture}}
\caption{System architecture of the \ours{} pipeline. Slotted candidate primitives derived from heterogeneous sources are assembled into a joint five-slot AOG structural space, optimized via an MCMC structural sampler against an internal record-grounded score, and rendered into verifiable normative prose.}
\label{fig:pipeline}
\end{figure}
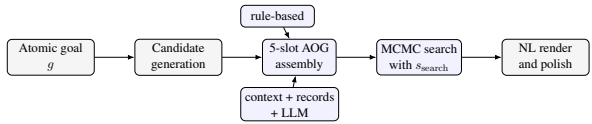

\paragraph{Candidate construction.}
For each operational slot \(k\), our framework builds a localized candidate repository \(\mathcal{C}_k\) from four distinct structural sources comprising symbolic rule templates, context-schema primitives, trace-derived event records, and LLM-augmented proposals. The complete structural search space is defined as the Cartesian product of these sets, formulated as
\[
\mathcal{A}(x)=\prod_{k\in K}\mathcal{C}_k(x).
\]
While the LLM-augmented fragments are highly valuable for injecting context-aware vocabulary and natural phrasing into rule pieces, they undergo a strict validation filter against the environment schema and active record logs to minimize the introduction of ungrounded or imaginary payloads. Across the 88 main evaluation goals, the deterministic symbolic tracks contribute an average of 14.0 template-backed candidates, 9.3 schema-derived primitives, and 10.2 record-backed traces per goal. The generative language model component contributes 16.7 additional fragments per configuration while the grounding filter discards an average of 1.3 unsupported candidates per execution, maintaining a stable 92.5\% candidate retention rate.

\paragraph{Structured search.}
A concrete state \(a\in\mathcal{A}(x)\) corresponds to a fully populated five-slot AOG structural layout. To navigate this combinatorial optimization space, we implement a discrete Markov Chain Monte Carlo, or MCMC, structural sampler driven by four specialized proposal distributions including single-slot replacement, parameter-only tuning, evidence-policy swapping, and coordinated trigger-evidence structural edits. The evaluation function linearly combines structural grounding indicators and linguistic quality features, penalizing hard operational feasibility violations according to the scoring function
\[
\ssearch(a;x)=w^\top\phi(a,x)-\lambda\,\mathbf{1}[\neg h(a,x)].
\]
The descriptive feature vector \(\phi(a,x)\) computationally operationalizes multiple programmatic dimensions covering intent alignment, schema grounding, record track compatibility, procedural closure, and value consistency, while \(h(a,x)\) acts as a binary predicate flagging execution failures. Given a proposed structural state change from \(a\) to \(a'\), the sampler accepts the mutation with a standard Metropolis-Hastings probability formulated as \(\min(1,\exp((\ssearch(a';x)-\ssearch(a;x))/\tau))\). This stochastic accept-reject gate provides the search mechanism with sufficient exploration freedom to escape local optima while consistently steering the trajectory toward record-grounded configurations. Our evaluation layout utilizes 200 optimization steps per run with the temperature parameter fixed at \(\tau=0.08\).

The optimization objective \(\ssearch\) functions as a structural prior and is decoupled from the downstream validation metric. It reflects domain-agnostic properties that any verifiable organizational clause should satisfy, prioritizing trace-backed evidence channels over unrecorded claims and structurally closed enforcement flows over ambiguous procedures. The parameter weights are calibrated on the 16 development goals and remain entirely frozen before test-set execution. Crucially, no weight is optimized against the test scenarios or the automatic evaluator outputs. Quantitative verification details regarding alternative uniform weightings, the ablation of hard constraints, and a baseline comparison using a mathematical CP-SAT solver~\citep{perron2023cp} are detailed in Appendix~\ref{app:result-suite}, isolating this structural search bias from test-time metric overfitting.
A worked AOG schematic and construction-and-search trace are provided in
Figure~\ref{fig:aog-case-study}.

\begin{table*}[!htbp]
\centering
\footnotesize
\resizebox{\textwidth}{!}{%
\begin{tabular}{lrrrrrrrrr}
\toprule
\textbf{Method} & \textbf{Spec} & \textbf{Exec} &
\textbf{Flu} & \textbf{Read} & \textbf{Faith} & \textbf{Avg} &
\textbf{v2 Avg} & \textbf{\hardok{}} & \textbf{\overall{}} \\
\midrule
Direct Prompting & 60.4 & 61.7 & 83.1 & \best{79.8} & 59.2 & 68.8 & 64.9 & 0.0 & 55.1 \\
Structured Outputs & 72.7 & 53.8 & 54.1 & 47.7 & 68.4 & 59.3 & 68.8 & 4.2 & 48.3 \\
Best-of-\(N\) & 72.0 & 52.7 & 54.5 & 47.4 & 68.7 & 59.1 & 69.0 & 6.1 & 48.5 \\
Self-Refine & 72.7 & 55.4 & 50.2 & 46.4 & 68.3 & 58.6 & 68.6 & 3.8 & 47.6 \\
In-Context Grounding & 70.2 & 51.9 & 59.8 & 49.9 & 67.8 & 59.9 & 69.7 & 0.0 & 48.0 \\
Norm Synthesis & 73.8 & 55.3 & 63.9 & 54.5 & 68.4 & 63.2 & 71.3 & 2.7 & 51.1 \\
No record grounding & 75.0 & 56.9 & 64.4 & 55.2 & 69.1 & 64.1 & 71.4 & 74.6 & 66.2 \\
No search & 74.0 & 56.7 & 64.5 & 56.1 & 67.0 & 63.7 & 71.7 & 46.2 & 60.2 \\
No LLM candidates & 75.6 & 57.7 & 74.4 & 60.5 & 68.1 & 67.2 & 70.5 & \best{77.3} & 69.2 \\
\textbf{\ours{}} & \best{80.2} & \best{73.3} &
\best{92.5} & 79.5 & \best{79.2} & \best{81.0} & \best{86.4} & 51.1 & \best{75.0} \\
\bottomrule
\end{tabular}%
}
\caption{Main test results on 88 goals \(\times\) 3 seeds. Rubric columns are
GPT-5.1 judgments on a 1--5 scale, reported as percentages. \textbf{Avg}
denotes Rubric Avg (v1), the equal-weight mean of the five rubric dimensions;
v2 Avg is a stricter diagnostic rubric that penalizes generic evidence
vocabulary. \hardok{} is a binary AOG feasibility diagnostic. \overall{} is the
disclosed executable composite \(0.80 \times\) Avg \(+ 0.20 \times\) \hardok{}.}
\label{tab:main}
\end{table*}

\paragraph{Realization and evaluator separation.}
The optimized latent graph layout is processed by a surface realization function \(R\) into an initial rule clause formulated as \(y=R(a^*)\), which is subsequently refined through a payload-preserving text polish workflow. A programmatic alignment guard monitors this editing phase, automatically rejecting any sentence rewrites that omit essential structural entities or discard required evidence logs. Crucially, the internal search scorer \(\ssearch\) is strictly isolated from the final performance evaluation setup. We report our headline achievements using \(\seval\), an independent multi-dimensional rubric score computed over the polished natural language text, alongside \hardok{}, a binary execution feasibility checker tracking true operational grounding within the environment records. This systematic boundary ensures that performance enhancements are driven by genuine structural grounding rather than superficial convergence toward the evaluation metric.

\section{Results and Analysis}
\label{sec:experiments}

\subsection{Evaluation Setup}
The in-distribution evaluation suite uses 88 distinct goals across 3 random seeds. Baselines encompass direct prompting, structured outputs~\citep{willard2023efficient}, best-of-\(N\) (\(N{=}8\)) reranking~\citep{nakano2021webgpt}, self-refinement~\citep{madaan2023self}, retrieval-style in-context grounding~\citep{lewis2020retrieval}, and a single-shot CRSEC norm-synthesis adaptation~\citep{ren2024emergence}, while pipeline ablations examine the removal of record grounding, structural search, and language model candidates. The CRSEC-style baseline adapts only the norm-creation and representation prompt. The primary scoring instrument is a blinded GPT-5.1 rubric evaluating five dimensions including specificity, executability, fluency, readability, and faithfulness. We report \textbf{Rubric Avg (v1)}, the equal-weight arithmetic mean of these five dimensions, Rubric Avg~(v1) = (Spec + Exec + Flu + Read + Faith)/5, as the main text-quality metric, and a stricter diagnostic \textbf{Rubric Avg (v2)} that additionally penalizes generic (non-specific) evidence references. We additionally report \hardok{} as a binary feasibility diagnostic over the selected AOG layout, alongside a disclosed executable composite score defined as \overall{} = 0.80 $\times$ Rubric Avg~(v1) + 0.20 $\times$ \hardok{}. Metric-weight sensitivity is reported in Appendix~\ref{app:metric-sensitivity}; human calibration and a second-model rubric check are reported in Appendix~\ref{app:full-results}, Tables~\ref{tab:human-calibration-app}--\ref{tab:second-evaluator-app}; rubric prompts are listed in Appendix~\ref{app:prompts}.

\subsection{Main Benchmark Results}
\label{sec:main-results}
Table~\ref{tab:main} shows the head-to-head performance on the main test split. \ours{} improves Rubric Avg (v1) to \pct{81.0}, achieving a notable margin over both direct prompting and the closest ablation framework. The primary source of this improvement stems from the executability dimension, where our approach raises the average performance to \pct{73.3}. Figure~\ref{fig:tradeoff-robustness} situates this on the Rubric-Avg-(v1) vs.\ \hardok{} plane. Under the standard Pareto definition, both full \ours{} and its No-LLM-candidates ablation are nondominated: full \ours{} attains the highest reader-facing quality, while No LLM candidates attains a higher \hardok{}. The No-LLM-candidates variant is an ablation of our method rather than an independent competitor---it retains the five-slot structure, record grounding, hard feasibility checking, and structured search---so its strong feasibility further supports the core framework; adding LLM candidates expands the candidate space and improves text quality at a different quality-feasibility balance. Our claim is therefore that full \ours{} achieves the highest score under the disclosed \overall{} composite (0.80 $\times$ quality + 0.20 $\times$ feasibility), not that it has the highest \hardok{}. Paired goal-level bootstrap analysis confirms that this composite ranking remains stable across alternative weights (full details in Appendix~\ref{app:metric-sensitivity}); shifting the weighting heavily toward the binary gate does reduce \hardok{}, confirming a genuine quality-feasibility tradeoff. Furthermore, our method-blinded human calibration study indicates strong directional alignment with the automatic scorer, showing a row-level Spearman correlation of 0.77 and yielding a positive human score margin of 0.69 over non-search alternatives.

\begin{table}[!htbp]
\centering
\footnotesize
\setlength{\tabcolsep}{4pt}
\begin{tabular}{lrrr}
\toprule
\textbf{Ablation} & \textbf{Avg} & \(\Delta\) & \textbf{\hardok{}} \\
\midrule
\ours{} full & 81.0 & 0.00 & 51.1 \\
Uniform weights & 80.4 & -0.53 & 60.6 \\
Random candidates & 79.9 & -1.03 & 36.0 \\
No hard penalty & 80.1 & -0.86 & 53.8 \\
100 search steps & 80.2 & -0.77 & 47.0 \\
AND-only evidence & 80.6 & -0.35 & 52.7 \\
CP-SAT control & 64.8 & -16.18 & 60.6 \\
\bottomrule
\end{tabular}
\caption{Ablations on \ours{}. \(\Delta\) is the Avg difference. All rows run on the same 88 goals and three seeds; CP-SAT uses the LLM-augmented candidate inventory with an OR-Tools backend.}
\label{tab:ablation}
\end{table}

\begin{figure*}[t]
\centering
\includegraphics[width=0.95\linewidth]{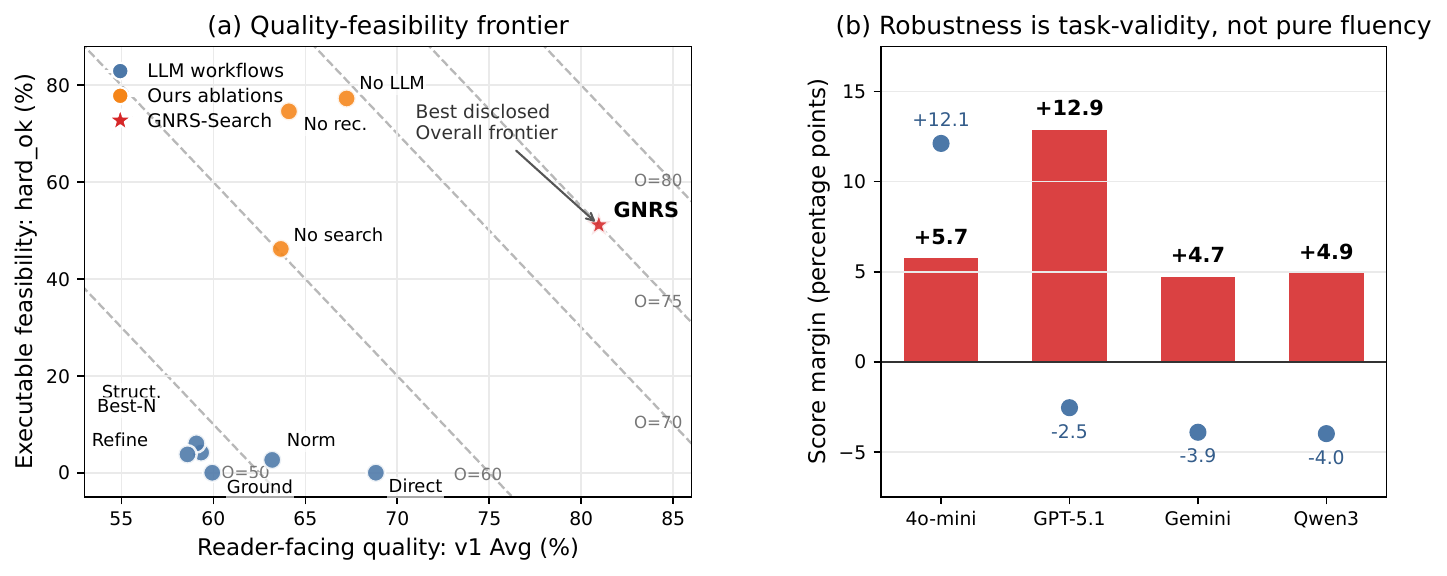}
\caption{Analytical view of the headline results. Left: Rubric Avg (v1) and \hardok{} form a quality-feasibility plane; dashed contours show the disclosed \overall{} composite, and points are labeled with abbreviated method names. Both full \ours{} and its No-LLM-candidates ablation are Pareto-optimal operating points: full \ours{} attains the highest reader-facing quality, while No LLM candidates attains higher \hardok{}. \ours{} achieves the highest score under the disclosed \overall{} composite, not the highest \hardok{}. Right: red bars show \ours{}'s executable-\overall{} margin over the best non-\ours{} method, while blue dots show \ours{} minus Direct Prompting on pure Rubric Avg (v1).}
\label{fig:tradeoff-robustness}
\end{figure*}

\subsection{Mechanism and Sensitivity}
\label{sec:ablation}

As shown in Table~\ref{tab:ablation}, removing individual algorithmic components degrades synthesis quality. Random candidate selection causes the largest absolute performance drop, which confirms that candidate ranking characteristics directly impact framework execution. Weakening the search space to 100 steps or removing the hard feasibility penalty also reduces overall performance, suggesting that the task rewards systematic structured decomposition over localized search heuristics. Crucially, the mathematical CP-SAT solver baseline performs substantially below our full framework on both average quality and composite metrics, demonstrating that the performance advantage cannot be replicated by simply replacing our sampling policy with an off-the-shelf discrete optimizer.

\begin{figure}[htbp]
\centering
\includegraphics[width=\linewidth]{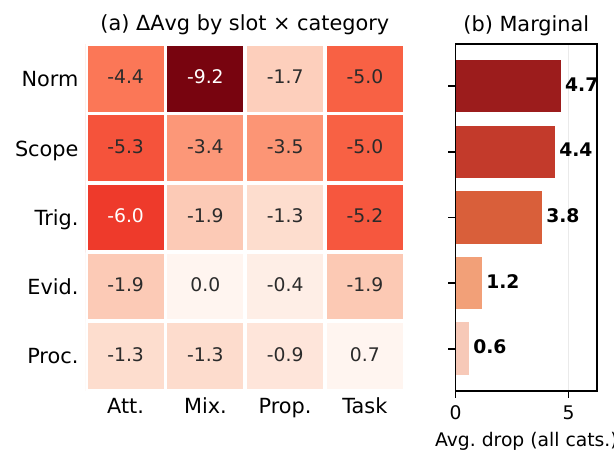}
\caption{Slot sensitivity. Each cell is the Avg change after replacing one slot's candidate pool with distractors from other goals.}
\label{fig:slot-sensitivity}
\end{figure}

To locate where synthesized rules remain most vulnerable, our slot sensitivity intervention degrades candidate sets one slot at a time using cross-goal distractors. Figure~\ref{fig:slot-sensitivity} rejects the initial hypothesis that evidence aggregation details would dictate overall fragility. Instead, quality loss concentrates heavily within the semantic core of the rule text, where degrading the norm body slot causes the largest average performance drop, followed closely by scope and trigger modifications. Conversely, procedure slots and evidence policy formulations remain highly resilient. Static pool audits confirm that this stability is not an artifact of small distractor sets, indicating that downstream verification properties are often structurally recoverable from environmental observables once the core compliance duties are correctly specified.

\subsection{Robustness and Real-Derived Policy Transfer}
\label{sec:robustness}

We evaluate generalization capabilities across a held-out scene family and an external cross-model generator grid. In the out-of-distribution experiment holding out open-space navigation goals, Table~\ref{tab:ood} confirms that \ours{} maintains its performance advantage by scoring \pct{79.8} on average quality and \pct{82.7} on the joint composite metric. Cross-model robustness evaluations summarized in Table~\ref{tab:crossmodel} show that our search framework ranks first under the composite task-validity metric across all tested commercial and open-source backbones. While open-ended direct prompting can achieve competitive surface text fluency on stronger model generators, it consistently yields zero binary feasibility scores due to its tendency to produce ungrounded generic clauses. This confirms that grounding-aware structural search remains necessary to satisfy joint executability requirements.

\paragraph{Human executability audit.}
To complement the programmatic \hardok{} check, we conduct a human executability audit on 48 tasks with three annotators. Under this audit, No LLM candidates achieves the highest strict executability rate at \pct{81.3}, while full \ours{} reaches \pct{75.0} and receives the highest overall human preference at \pct{68.8} because of its stronger completeness and rule quality. \hardok{} shows a per-task agreement of 0.64 with human executability judgments. These results support our claim that full \ours{} offers the best overall balance rather than the highest standalone feasibility, and corroborate that \hardok{} tracks human feasibility judgments to a moderate degree.

\paragraph{Robustness to schema noise.}
We further evaluate \ours{} under degraded context schemas: irrelevant noise causes a 2.1-point Rubric Avg (v1) drop, moderate missing information causes a 5.4-point drop, and incorrectly specified schemas cause a larger 9.7-point drop. The method is thus relatively stable under irrelevant noise and partial information, but sensitive to incorrect operational inputs, which suggests that schema validation is an important deployment requirement (Appendix~\ref{app:full-results}).

\begin{table}[!htbp]
\centering
\footnotesize
\setlength{\tabcolsep}{4pt}
\resizebox{\columnwidth}{!}{%
\begin{tabular}{lrrrr}
\toprule
\textbf{Method} & \textbf{Avg} & \textbf{\hardok{}} & \textbf{\overall{}} & \(\ssearch\) \\
\midrule
Direct & 65.6 & 0.0 & 52.4 & 0.546 \\
Structured & 58.0 & 0.0 & 46.4 & 0.805 \\
Best-of-\(N\) & 58.1 & 0.0 & 46.5 & 0.814 \\
Self-Refine & 58.2 & 0.0 & 46.6 & 0.806 \\
Grounding & 52.7 & 0.0 & 42.1 & 0.803 \\
Norm Synth. & 57.8 & 0.0 & 46.2 & 0.790 \\
No record & 57.7 & \best{100.0} & 66.1 & 0.995 \\
No search & 56.7 & 86.1 & 62.6 & 0.888 \\
No LLM & 60.9 & \best{100.0} & 68.7 & 0.917 \\
\textbf{\ours{}} & \best{79.8} & 94.4 & \best{82.7} & 0.904 \\
\bottomrule
\end{tabular}}
\caption{OOD results on the held-out open-space navigation family (12 goals \(\times\) 3 seeds). \overall{} uses the same disclosed \(0.80/0.20\) weighting as Table~\ref{tab:main}. Method names are abbreviated in the table body for width.}
\label{tab:ood}
\end{table}

\begin{table}[!htbp]
\centering
\footnotesize
\setlength{\tabcolsep}{4pt}
\resizebox{\columnwidth}{!}{%
\begin{tabular}{lrrrrr}
\toprule
\textbf{Generator} & \textbf{Direct} & \textbf{Refine} & \textbf{No LLM} & \textbf{\ours{}} & \(\Delta\) \\
\midrule
GPT-4o-mini & 55.1 & 47.6 & 69.2 & \best{75.0} & +5.7 \\
GPT-5.1 & 67.3 & 66.2 & 69.1 & \best{82.0} & +12.9 \\
Gemini Flash Lite & 61.6 & 51.7 & 69.0 & \best{73.7} & +4.7 \\
Qwen3-235B & 63.1 & 54.9 & 69.2 & \best{74.1} & +4.9 \\
\bottomrule
\end{tabular}}
\caption{Cross-model executable \overall{} with a fixed GPT-5.1 scorer. \(\Delta\) is \ours{} minus the best non-\ours{} row for that generator.}
\label{tab:crossmodel}
\end{table}

External evaluation on RealCharter-Bench indicates that our structured optimization translates effectively to authentic public policies. As reported in Table~\ref{tab:realcharter}, \ours{} achieves an average text quality score of \pct{93.4} alongside a perfect feasibility rate, leading the strongest iterative self-refinement baseline. Figure~\ref{fig:external-transfer} details this relationship across the quality-feasibility plane, demonstrating that the external transfer performance gains are driven by structured search optimization rather than template formatting effects. Crucially, the perfect feasibility alignment reflects the explicit enumeration of target slots within normalized task cards, which reduces the likelihood of introducing unsupported environment payloads compared to open-ended simulation settings. 

\begin{table}[htbp]
\centering
\footnotesize
\setlength{\tabcolsep}{4pt}
\begin{tabular}{lrrrr}
\toprule
\textbf{Method} & \textbf{n} & \textbf{Faith} & \textbf{\hardok{}} & \textbf{\overall{}} \\
\midrule
Direct & 53 & 4.64 & 62.26 & 82.26 \\
Self-Refine & 53 & 4.77 & 77.36 & 86.09 \\
No LLM & 53 & 4.50 & 47.17 & 76.07 \\
No search & 53 & 2.68 & 0.00 & 50.56 \\
\textbf{\ours{}} & 53 & \best{5.00} & \best{100.00} & \best{94.69} \\
\bottomrule
\end{tabular}
\caption{RealCharter-Bench results. Faith is the 1--5 task-faithfulness score; \hardok{} and \overall{} are percentages. All methods receive the same normalized task card, while the source clause text is hidden.}
\label{tab:realcharter}
\end{table}

In summary, our empirical analysis yields three clear findings regarding verifiable rule synthesis. First, fluent open-ended prompting remains an unreliable indicator of actual compliance grounding due to hidden verification failures. Second, enforcing syntactic formats via schema constraints or rhetorical self-refinement is insufficient without grounding intermediate components against empirical logs. Third, the localized slot sensitivity results demonstrate that the core bottleneck in policy synthesis resides within the interlocking definitions of norm bodies, scopes, and triggers, rather than downstream procedural closures.

\begin{figure}[htbp]
\centering
\includegraphics[width=\linewidth]{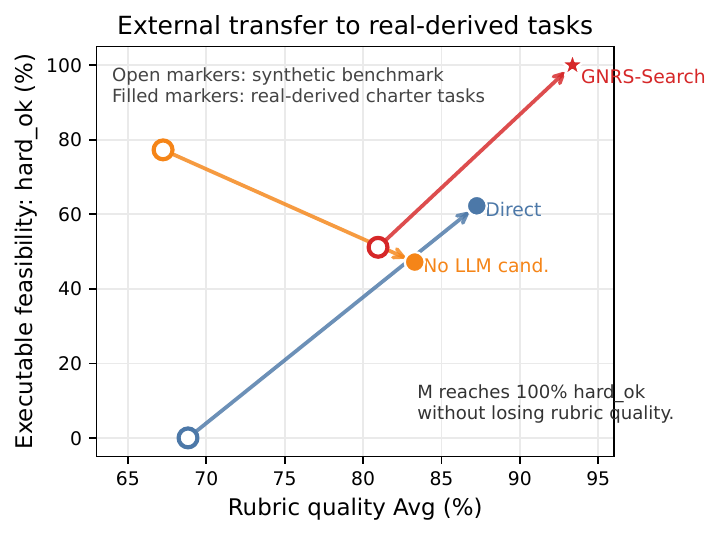}
\caption{External transfer from the synthetic benchmark to real-derived charter tasks for three anchor methods. Open markers show the main synthetic benchmark; filled markers show the 53 real-derived tasks.}
\label{fig:external-transfer}
\end{figure}

\section{Conclusion}

We formalize Grounded Normative Rule Synthesis alongside two evaluation benchmarks comprising the controlled \dataset{} and the real-derived RealCharter-Bench. To address the operational grounding bottleneck, we introduce \ours{}, an optimization framework that decouples latent structural commitments from surface prose generation through discrete MCMC sampling over an AOG layout. Empirical evaluations demonstrate that this paradigm outperforms traditional direct prompting and iterative self-refinement baselines across varying generator backbones and out-of-distribution settings. Systematic sensitivity audits further confirm that these performance gains come from securing the semantic core of the policies rather than optimizing superficial textual phrasing. This work establishes a verifiable framework for deploying compliance-ready and auditable personal agents within regulated environments.

\section*{Limitations}

\dataset{} is intentionally controlled: it evaluates grounded rule synthesis
against a fixed context schema and does not cover the full diversity of
institutional policies, legal domains, languages, or deployment environments.
RealCharter-Bench broadens the evaluation to public policy materials, but its
inputs are normalized policy cards rather than complete source documents, so it
does not test long-document interpretation or end-to-end policy extraction. Our
rubric scores, second-model audit, cross-model generator grid, and
method-blinded human calibration provide robustness checks, but they are not a
substitute for a broad evaluator-family study or domain-expert legal review.
Finally, GNRS-Search assumes that the relevant operational vocabulary and
evidence interfaces are available before generation. In many target
environments, this information already exists in system assets such as database
schemas, audit logs, access-control models, and workflow definitions; the context
schema and record layer can therefore be mapped from these existing assets rather
than authored from scratch by ordinary users. This preparation is mainly a
one-time environment-level cost that can subsequently support many
rule-generation tasks, whereas manually drafting grounded rules usually requires
repeated effort for each individual rule. Real deployments would still require
independent validation of the task cards, evidence interfaces, generated clauses,
and enforcement procedures, particularly when the provided schemas are
incorrectly specified, since our robustness analysis (Appendix~\ref{app:full-results}) shows that the method is relatively stable under irrelevant noise and partial missing information but sensitive to incorrect operational inputs.

\section*{AI Assistant Use}

All research ideas, experimental designs, benchmark construction, evaluations,
analyses, claims, and conclusions were developed and conducted by the authors.
AI assistants were used only for language polishing and editing support.

\bibliography{custom}

\begin{thebibliography}{45}
\providecommand{\natexlab}[1]{#1}

\bibitem[{Boella et~al.(2006)Boella, Van Der~Torre, and
  Verhagen}]{boella2006introduction}
Guido Boella, Leendert Van Der~Torre, and Harko Verhagen. 2006.
\newblock Introduction to normative multiagent systems.
\newblock \emph{Computational \& Mathematical Organization Theory},
  12(2):71--79.

\bibitem[{Cobbe et~al.(2021)Cobbe, Kosaraju, Bavarian, Chen, Jun, Kaiser,
  Plappert, Tworek, Hilton, Nakano et~al.}]{cobbe2021training}
Karl Cobbe, Vineet Kosaraju, Mohammad Bavarian, Mark Chen, Heewoo Jun, Lukasz
  Kaiser, Matthias Plappert, Jerry Tworek, Jacob Hilton, Reiichiro Nakano, and
  1 others. 2021.
\newblock Training verifiers to solve math word problems.
\newblock \emph{arXiv preprint arXiv:2110.14168}.

\bibitem[{Dong et~al.(2025)Dong, Ruan, Cai, Xu, Zhao, Lai, and
  Chen}]{dong2025xgrammar}
Yixin Dong, Charlie~F Ruan, Yaxing Cai, Ziyi Xu, Yilong Zhao, Ruihang Lai, and
  Tianqi Chen. 2025.
\newblock Xgrammar: Flexible and efficient structured generation engine for
  large language models.
\newblock \emph{Proceedings of Machine Learning and Systems}, 7.

\bibitem[{Gardent et~al.(2017)Gardent, Shimorina, Narayan, and
  Perez-Beltrachini}]{gardent2017webnlg}
Claire Gardent, Anastasia Shimorina, Shashi Narayan, and Laura
  Perez-Beltrachini. 2017.
\newblock The webnlg challenge: Generating text from rdf data.
\newblock In \emph{Proceedings of the 10th international conference on natural
  language generation}, pages 124--133.

\bibitem[{Governatori and Rotolo(2010)}]{governatori2010conceptually}
Guido Governatori and Antonino Rotolo. 2010.
\newblock A conceptually rich model of business process compliance.
\newblock In \emph{Proceedings of the Seventh Asia-Pacific Conference on
  Conceptual Modelling-Volume 110}, pages 3--12.

\bibitem[{Guha et~al.(2023)Guha, Nyarko, Ho, R{\'e}, Chilton, Chohlas-Wood,
  Peters, Waldon, Rockmore, Zambrano et~al.}]{guha2023legalbench}
Neel Guha, Julian Nyarko, Daniel Ho, Christopher R{\'e}, Adam Chilton, Alex
  Chohlas-Wood, Austin Peters, Brandon Waldon, Daniel Rockmore, Diego Zambrano,
  and 1 others. 2023.
\newblock Legalbench: A collaboratively built benchmark for measuring legal
  reasoning in large language models.
\newblock \emph{Advances in neural information processing systems},
  36:44123--44279.

\bibitem[{Hastings(1970)}]{hastings1970monte}
W~Keith Hastings. 1970.
\newblock Monte carlo sampling methods using markov chains and their
  applications.

\bibitem[{Hendrycks et~al.(2021)Hendrycks, Burns, Chen, and
  Ball}]{hendrycks2021cuad}
Dan Hendrycks, Collin Burns, Anya Chen, and Spencer Ball. 2021.
\newblock Cuad: An expert-annotated nlp dataset for legal contract review.
\newblock \emph{arXiv preprint arXiv:2103.06268}.

\bibitem[{Holzenberger et~al.(2020)Holzenberger, Blair-Stanek, and
  Van~Durme}]{holzenberger2020dataset}
Nils Holzenberger, Andrew Blair-Stanek, and Benjamin Van~Durme. 2020.
\newblock A dataset for statutory reasoning in tax law entailment and question
  answering.
\newblock \emph{arXiv preprint arXiv:2005.05257}.

\bibitem[{Huang et~al.(2024)Huang, Wang, Liu, Kong, Qin, Tang, Zhu, Bi, Qi, and
  Feng}]{huang2024adasociety}
Yizhe Huang, Xingbo Wang, Hao Liu, Fanqi Kong, Aoyang Qin, Min Tang, Song-Chun
  Zhu, Mingjie Bi, Siyuan Qi, and Xue Feng. 2024.
\newblock Adasociety: An adaptive environment with social structures for
  multi-agent decision-making.
\newblock \emph{Advances in Neural Information Processing Systems},
  37:35388--35413.

\bibitem[{Ji et~al.(2023)Ji, Lee, Frieske, Yu, Su, Xu, Ishii, Bang, Madotto,
  and Fung}]{ji2023survey}
Ziwei Ji, Nayeon Lee, Rita Frieske, Tiezheng Yu, Dan Su, Yan Xu, Etsuko Ishii,
  Ye~Jin Bang, Andrea Madotto, and Pascale Fung. 2023.
\newblock Survey of hallucination in natural language generation.
\newblock \emph{ACM computing surveys}, 55(12):1--38.

\bibitem[{Jiayang et~al.(2025)Jiayang, Zhuang, Li, Chan, Liu, Qiu, and
  Song}]{jiayang2025integround}
Cheng Jiayang, Qianqian Zhuang, Haoran Li, Chunkit Chan, Xin Liu, Lin Qiu, and
  Yangqiu Song. 2025.
\newblock Integround: On the evaluation of verification and retrieval planning
  in integrative grounding.
\newblock In \emph{Findings of the Association for Computational Linguistics:
  EMNLP 2025}, pages 13587--13602.

\bibitem[{Kong et~al.(2024)Kong, Huang, Zhu, Qi, and Feng}]{kong2024learning}
Fanqi Kong, Yizhe Huang, Song-Chun Zhu, Siyuan Qi, and Xue Feng. 2024.
\newblock Learning to balance altruism and self-interest based on empathy in
  mixed-motive games.
\newblock \emph{Advances in Neural Information Processing Systems},
  37:135819--135842.

\bibitem[{Kong et~al.(2026{\natexlab{a}})Kong, Zhang, Yin, Zhang, Zhang, Chen,
  Zhang, Zhang, Zhu, and Feng}]{kong2026aegis}
Fanqi Kong, Ruijie Zhang, Huaxiao Yin, Guibin Zhang, Xiaofei Zhang, Ziang Chen,
  Zhaowei Zhang, Xiaoyuan Zhang, Song-Chun Zhu, and Xue Feng.
  2026{\natexlab{a}}.
\newblock Aegis: Automated error generation and attribution for multi-agent
  systems.
\newblock In \emph{International Conference on Learning Representations},
  volume 2026, pages 52994--53028.

\bibitem[{Kong et~al.(2025)Kong, Zhang, Chen, Yang, Zhu, and
  Feng}]{kong2025enhancing}
Fanqi Kong, Xiaoyuan Zhang, Xinyu Chen, Yaodong Yang, Song-Chun Zhu, and Xue
  Feng. 2025.
\newblock Enhancing llm-based social bot via an adversarial learning framework.
\newblock In \emph{Proceedings of the 2025 Conference on Empirical Methods in
  Natural Language Processing}, pages 23246--23271.

\bibitem[{Kong et~al.(2026{\natexlab{b}})Kong, Zu, Chen, Yang, Zhu, and
  Feng}]{kong2026siv}
Fanqi Kong, Weiqin Zu, Xinyu Chen, Yaodong Yang, Song-Chun Zhu, and Xue Feng.
  2026{\natexlab{b}}.
\newblock Siv-bench: A video benchmark for social interaction understanding and
  reasoning.
\newblock In \emph{Findings of the Association for Computational Linguistics:
  ACL 2026}, pages 37379--37403.

\bibitem[{Lewis et~al.(2020)Lewis, Perez, Piktus, Petroni, Karpukhin, Goyal,
  K{\"u}ttler, Lewis, Yih, Rockt{\"a}schel et~al.}]{lewis2020retrieval}
Patrick Lewis, Ethan Perez, Aleksandra Piktus, Fabio Petroni, Vladimir
  Karpukhin, Naman Goyal, Heinrich K{\"u}ttler, Mike Lewis, Wen-tau Yih, Tim
  Rockt{\"a}schel, and 1 others. 2020.
\newblock Retrieval-augmented generation for knowledge-intensive nlp tasks.
\newblock \emph{Advances in neural information processing systems},
  33:9459--9474.

\bibitem[{Liu et~al.(2024)Liu, Yu, Zhang, Xu, Lei, Lai, Gu, Ding, Men, Yang
  et~al.}]{liu2024agentbench}
Xiao Liu, Hao Yu, Hanchen Zhang, Yifan Xu, Xuanyu Lei, Hanyu Lai, Yu~Gu,
  Hangliang Ding, Kaiwen Men, Kejuan Yang, and 1 others. 2024.
\newblock Agentbench: Evaluating llms as agents.
\newblock In \emph{International Conference on Learning Representations},
  volume 2024, pages 52989--53046.

\bibitem[{Lu et~al.(2025)Lu, Li, Cong, Zhang, Wu, Lin, Liu, Liu, and
  Sun}]{lu2025learning}
Yaxi Lu, Haolun Li, Xin Cong, Zhong Zhang, Yesai Wu, Yankai Lin, Zhiyuan Liu,
  Fangming Liu, and Maosong Sun. 2025.
\newblock Learning to generate structured output with schema reinforcement
  learning.
\newblock In \emph{Proceedings of the 63rd Annual Meeting of the Association
  for Computational Linguistics (Volume 1: Long Papers)}, pages 4905--4918.

\bibitem[{Madaan et~al.(2023)Madaan, Tandon, Gupta, Hallinan, Gao, Wiegreffe,
  Alon, Dziri, Prabhumoye, Yang et~al.}]{madaan2023self}
Aman Madaan, Niket Tandon, Prakhar Gupta, Skyler Hallinan, Luyu Gao, Sarah
  Wiegreffe, Uri Alon, Nouha Dziri, Shrimai Prabhumoye, Yiming Yang, and 1
  others. 2023.
\newblock Self-refine: Iterative refinement with self-feedback.
\newblock \emph{Advances in neural information processing systems},
  36:46534--46594.

\bibitem[{Maynez et~al.(2020)Maynez, Narayan, Bohnet, and
  McDonald}]{maynez2020faithfulness}
Joshua Maynez, Shashi Narayan, Bernd Bohnet, and Ryan McDonald. 2020.
\newblock On faithfulness and factuality in abstractive summarization.
\newblock In \emph{Proceedings of the 58th annual meeting of the association
  for computational linguistics}, pages 1906--1919.

\bibitem[{Nakano et~al.(2021)Nakano, Hilton, Balaji, Wu, Ouyang, Kim, Hesse,
  Jain, Kosaraju, Saunders et~al.}]{nakano2021webgpt}
Reiichiro Nakano, Jacob Hilton, Suchir Balaji, Jeff Wu, Long Ouyang, Christina
  Kim, Christopher Hesse, Shantanu Jain, Vineet Kosaraju, William Saunders, and
  1 others. 2021.
\newblock Webgpt: Browser-assisted question-answering with human feedback.
\newblock \emph{arXiv preprint arXiv:2112.09332}.

\bibitem[{Nan et~al.(2021)Nan, Radev, Zhang, Rau, Sivaprasad, Hsieh, Tang,
  Vyas, Verma, Krishna et~al.}]{nan2021dart}
Linyong Nan, Dragomir Radev, Rui Zhang, Amrit Rau, Abhinand Sivaprasad,
  Chiachun Hsieh, Xiangru Tang, Aadit Vyas, Neha Verma, Pranav Krishna, and 1
  others. 2021.
\newblock Dart: Open-domain structured data record to text generation.
\newblock In \emph{Proceedings of the 2021 Conference of the North American
  Chapter of the Association for Computational Linguistics: Human Language
  Technologies}, pages 432--447.

\bibitem[{{OpenAI}(2024)}]{openai2024structured}
{OpenAI}. 2024.
\newblock Introducing {S}tructured {O}utputs in the {API}.
\newblock Blog post, August 2024.

\bibitem[{Parikh et~al.(2020)Parikh, Wang, Gehrmann, Faruqui, Dhingra, Yang,
  and Das}]{parikh2020totto}
Ankur Parikh, Xuezhi Wang, Sebastian Gehrmann, Manaal Faruqui, Bhuwan Dhingra,
  Diyi Yang, and Dipanjan Das. 2020.
\newblock Totto: A controlled table-to-text generation dataset.
\newblock In \emph{Proceedings of the 2020 Conference on Empirical Methods in
  Natural Language Processing (EMNLP)}, pages 1173--1186.

\bibitem[{Park et~al.(2023)Park, O'Brien, Cai, Morris, Liang, and
  Bernstein}]{park2023generative}
Joon~Sung Park, Joseph O'Brien, Carrie~Jun Cai, Meredith~Ringel Morris, Percy
  Liang, and Michael~S Bernstein. 2023.
\newblock Generative agents: Interactive simulacra of human behavior.
\newblock In \emph{Proceedings of the 36th annual acm symposium on user
  interface software and technology}, pages 1--22.

\bibitem[{Perron et~al.(2023)Perron, Didier, and Gay}]{perron2023cp}
Laurent Perron, Fr{\'e}d{\'e}ric Didier, and Steven Gay. 2023.
\newblock The cp-sat-lp solver (invited talk).
\newblock In \emph{29th International Conference on Principles and Practice of
  Constraint Programming (CP 2023)}, pages 3--1. Schloss
  Dagstuhl--Leibniz-Zentrum f{\"u}r Informatik.

\bibitem[{Ren et~al.(2024)Ren, Cui, Song, Wang, and Hu}]{ren2024emergence}
Siyue Ren, Zhiyao Cui, Ruiqi Song, Zhen Wang, and Shuyue Hu. 2024.
\newblock Emergence of social norms in generative agent societies: principles
  and architecture.
\newblock \emph{arXiv preprint arXiv:2403.08251}.

\bibitem[{Scholak et~al.(2021)Scholak, Schucher, and
  Bahdanau}]{scholak2021picard}
Torsten Scholak, Nathan Schucher, and Dzmitry Bahdanau. 2021.
\newblock Picard: Parsing incrementally for constrained auto-regressive
  decoding from language models.
\newblock In \emph{Proceedings of the 2021 conference on empirical methods in
  natural language processing}, pages 9895--9901.

\bibitem[{Sergot et~al.(1986)Sergot, Sadri, Kowalski, Kriwaczek, Hammond, and
  Cory}]{sergot1986british}
Marek~J. Sergot, Fariba Sadri, Robert~A. Kowalski, Frank Kriwaczek, Peter
  Hammond, and H~Terese Cory. 1986.
\newblock The british nationality act as a logic program.
\newblock \emph{Communications of the ACM}, 29(5):370--386.

\bibitem[{Shinn et~al.(2023)Shinn, Cassano, Gopinath, Narasimhan, and
  Yao}]{shinn2023reflexion}
Noah Shinn, Federico Cassano, Ashwin Gopinath, Karthik Narasimhan, and Shunyu
  Yao. 2023.
\newblock Reflexion: Language agents with verbal reinforcement learning.
\newblock \emph{Advances in neural information processing systems},
  36:8634--8652.

\bibitem[{Shridhar et~al.(2020)Shridhar, Yuan, C{\^o}t{\'e}, Bisk, Trischler,
  and Hausknecht}]{shridhar2020alfworld}
Mohit Shridhar, Xingdi Yuan, Marc-Alexandre C{\^o}t{\'e}, Yonatan Bisk, Adam
  Trischler, and Matthew Hausknecht. 2020.
\newblock Alfworld: Aligning text and embodied environments for interactive
  learning.
\newblock \emph{arXiv preprint arXiv:2010.03768}.

\bibitem[{Solar-Lezama(2008)}]{solar2008program}
Armando Solar-Lezama. 2008.
\newblock \emph{Program synthesis by sketching}.
\newblock University of California, Berkeley.

\bibitem[{Solar-Lezama et~al.(2006)Solar-Lezama, Tancau, Bodik, Seshia, and
  Saraswat}]{solar2006combinatorial}
Armando Solar-Lezama, Liviu Tancau, Rastislav Bodik, Sanjit Seshia, and Vijay
  Saraswat. 2006.
\newblock Combinatorial sketching for finite programs.
\newblock In \emph{Proceedings of the 12th international conference on
  Architectural support for programming languages and operating systems}, pages
  404--415.

\bibitem[{Tam et~al.(2024)Tam, Wu, Tsai, Lin, Lee, and Chen}]{tam2024let}
Zhi~Rui Tam, Cheng-Kuang Wu, Yi-Lin Tsai, Chieh-Yen Lin, Hung-yi Lee, and
  Yun-Nung Chen. 2024.
\newblock Let me speak freely? a study on the impact of format restrictions on
  large language model performance.
\newblock In \emph{Proceedings of the 2024 Conference on Empirical Methods in
  Natural Language Processing: Industry Track}, pages 1218--1236.

\bibitem[{Wang et~al.(2023)Wang, Xie, Jiang, Mandlekar, Xiao, Zhu, Fan, and
  Anandkumar}]{wang2023voyager}
Guanzhi Wang, Yuqi Xie, Yunfan Jiang, Ajay Mandlekar, Chaowei Xiao, Yuke Zhu,
  Linxi Fan, and Anima Anandkumar. 2023.
\newblock Voyager: An open-ended embodied agent with large language models.
\newblock \emph{arXiv preprint arXiv:2305.16291}.

\bibitem[{Willard and Louf(2023)}]{willard2023efficient}
Brandon~T Willard and R{\'e}mi Louf. 2023.
\newblock Efficient guided generation for large language models.
\newblock \emph{arXiv preprint arXiv:2307.09702}.

\bibitem[{Wiseman et~al.(2017)Wiseman, Shieber, and
  Rush}]{wiseman2017challenges}
Sam Wiseman, Stuart~M Shieber, and Alexander~M Rush. 2017.
\newblock Challenges in data-to-document generation.
\newblock In \emph{Proceedings of the 2017 conference on empirical methods in
  natural language processing}, pages 2253--2263.

\bibitem[{Wu and Zhu(2011)}]{wu2011numerical}
Tianfu Wu and Song-Chun Zhu. 2011.
\newblock A numerical study of the bottom-up and top-down inference processes
  in and-or graphs.
\newblock \emph{International journal of computer vision}, 93(2):226--252.

\bibitem[{Yao et~al.(2022)Yao, Zhao, Yu, Du, Shafran, Narasimhan, and
  Cao}]{yao2022react}
Shunyu Yao, Jeffrey Zhao, Dian Yu, Nan Du, Izhak Shafran, Karthik Narasimhan,
  and Yuan Cao. 2022.
\newblock React: Synergizing reasoning and acting in language models.
\newblock \emph{arXiv preprint arXiv:2210.03629}.

\bibitem[{Zhang et~al.(2025)Zhang, Peng, Kong, Yang, Wu, Yu, Xiang, Ruan, Wang,
  Song et~al.}]{zhang2025autoenv}
Jiayi Zhang, Yiran Peng, Fanqi Kong, Cheng Yang, Yifan Wu, Zhaoyang Yu, Jinyu
  Xiang, Jianhao Ruan, Jinlin Wang, Maojia Song, and 1 others. 2025.
\newblock Autoenv: Automated environments for measuring cross-environment agent
  learning.
\newblock \emph{arXiv preprint arXiv:2511.19304}.

\bibitem[{Zhang et~al.(2026{\natexlab{a}})Zhang, Huang, Ma, Chen, Ma, Du, Zhu,
  Yang, and Feng}]{zhang2026social}
Xiaoyuan Zhang, Yizhe Huang, Chengdong Ma, Zhixun Chen, Long Ma, Yali Du,
  Song-Chun Zhu, Yaodong Yang, and Xue Feng. 2026{\natexlab{a}}.
\newblock Social world model-augmented mechanism design policy learning.
\newblock \emph{Advances in Neural Information Processing Systems},
  38:111699--111720.

\bibitem[{Zhang et~al.(2026{\natexlab{b}})Zhang, Ma, Huang, Huang, Qi, Zhu,
  Feng, and Yang}]{zhang2026world}
Xiaoyuan Zhang, Chengdong Ma, Yizhe Huang, Weidong Huang, Siyuan Qi, Song-Chun
  Zhu, Xue Feng, and Yaodong Yang. 2026{\natexlab{b}}.
\newblock World models should prioritize the unification of physical and social
  dynamics.
\newblock \emph{Advances in Neural Information Processing Systems}, 38.

\bibitem[{Zhou et~al.(2025)Zhou, Hua, Pan, Cheng, Wu, Yu, and
  Wang}]{zhou2025rulearena}
Ruiwen Zhou, Wenyue Hua, Liangming Pan, Sitao Cheng, Xiaobao Wu, En~Yu, and
  William~Yang Wang. 2025.
\newblock Rulearena: A benchmark for rule-guided reasoning with llms in
  real-world scenarios.
\newblock In \emph{Proceedings of the 63rd Annual Meeting of the Association
  for Computational Linguistics (Volume 1: Long Papers)}, pages 550--572.

\bibitem[{Zhu and Mumford(2007)}]{zhu2007stochastic}
Song-Chun Zhu and David Mumford. 2007.
\newblock A stochastic grammar of images.
\newblock \emph{Foundations and trends in computer graphics and vision},
  2(4):259--362.

\end{thebibliography}

\appendix

\section{Benchmark Details}
\label{app:benchmark-stats}

This appendix records the explicit goal taxonomy of \dataset{} and the document
lineage of RealCharter-Bench. The tables below summarize the released benchmark
artifacts at the level needed to reproduce dataset composition without rerunning
the generation or evaluation pipeline.

\paragraph{Artifact use and intended use.}
Existing artifacts are used only as research evaluation substrates: public
policy materials are converted into normalized task cards with source clauses
hidden from generators, and the environment-derived record layer defines typed
actions, observables, and evidence channels. Where source materials specify
access or license conditions, we retain attribution and provenance and limit
derived benchmark materials to research evaluation. The artifacts introduced in
this paper are intended for studying grounded rule synthesis and evaluation, not
for legal advice, institutional policy adoption, or automated enforcement without
independent review.

\subsection{\dataset{} Scene Families and Goal Splits}

\dataset{} contains 116 goals organized along two orthogonal axes. The first
axis is the \emph{scene family}, which determines the operational vocabulary
(roles, objects, canonical actions, observables, and evidence channels) that
the rule must be grounded against. The second axis is the \emph{goal category},
which controls the normative pressure the rule must address. Family sizes vary
between 12 and 18 by scenic affordance, while the four goal categories each
contain exactly 29 goals.

\begin{table*}[!t]
\centering
\small
\setlength{\tabcolsep}{4pt}
\begin{tabular}{lp{5.4cm}rrrrr}
\toprule
\textbf{Scene family} & \textbf{Scene tags} &
\textbf{Att.} & \textbf{Prop.} & \textbf{Task.} & \textbf{Mix.} & \textbf{Total} \\
\midrule
service\_coordination       & \texttt{restaurant}                                  & 4 & 4 & 4 & 4 & 16 \\
public\_order\_and\_patrol  & \texttt{public\_square, territory\_border}           & 4 & 4 & 4 & 4 & 16 \\
commerce\_and\_exchange     & \texttt{shop, public\_square}                         & 4 & 3 & 4 & 3 & 14 \\
residence\_and\_rest        & \texttt{territory\_border, public\_square}            & 3 & 4 & 3 & 4 & 14 \\
labor\_and\_production      & \texttt{farm}                                         & 4 & 3 & 4 & 3 & 14 \\
education\_and\_play        & \texttt{public\_square}                               & 3 & 3 & 3 & 3 & 12 \\
crisis\_response            & \texttt{public\_square, restaurant, farm}             & 4 & 5 & 5 & 4 & 18 \\
open\_space\_navigation \textit{(OOD)} & \texttt{public\_square, territory\_border} & 3 & 3 & 2 & 4 & 12 \\
\midrule
\textbf{Total}              & ---                                                   & \textbf{29} & \textbf{29} & \textbf{29} & \textbf{29} & \textbf{116} \\
\bottomrule
\end{tabular}
\caption{Per-family goal counts crossed with the four goal categories. The
\textit{open\_space\_navigation} family is held out for OOD evaluation; the
remaining 104 goals are split into 16 development and 88 in-distribution test
goals.}
\label{tab:bench-family-cat}
\end{table*}

Table~\ref{tab:bench-cat-examples} grounds the four category labels in concrete
goal instances. The example identifier names a benchmark item, and the intent
stem gives the natural-language policy objective behind that item.

\begin{table*}[!t]
\centering
\footnotesize
\setlength{\tabcolsep}{4pt}
\begin{tabular}{lp{2.4cm}p{4.4cm}p{4.0cm}}
\toprule
\textbf{Category} & \textbf{Example} \textsc{goal\_id} & \textbf{Intent stem} &
\textbf{Required / forbidden actions} \\
\midrule
attendance     & GNRS-001 \newline \textit{service\_coordination} &
coordinate shared service shifts with auditable arrival and closure &
req.\ \texttt{CheckIn}, \texttt{CheckOut}; forbid.\ none \\
property       & GNRS-059 \newline \textit{residence\_and\_rest} &
protect rest and personal boundaries while allowing justified access &
req.\ \texttt{Speak}; forbid.\ \texttt{Interact} \\
task\_execution & GNRS-002 \newline \textit{service\_coordination} &
coordinate shared service shifts with auditable arrival and closure &
req.\ \texttt{CheckOut}, \texttt{Deliver}; forbid.\ none \\
mixed          & GNRS-099 \newline \textit{crisis\_response} &
make emergency response fast, auditable, and fair under resource pressure &
req.\ \texttt{Speak}; forbid.\ \texttt{Trade} \\
\bottomrule
\end{tabular}
\caption{One illustrative goal per category. Required and forbidden actions are
drawn from the canonical action set of Table~\ref{tab:bench-record-layer}; full
roles, objects, observables, and evidence channels are part of the released
benchmark specification.}
\label{tab:bench-cat-examples}
\end{table*}

\subsection{Operational Record Layer Vocabulary}

The record layer \(D\) is a read-only projection from the environment substrate:
286 child-action definitions, 37 NPC action names, 16 composite actions, 3{,}285
resource instances, and 38 rooms are normalized into the typed vocabulary in
Table~\ref{tab:bench-record-layer}. Generators receive the projected vocabulary
only. Each
canonical action carries a \texttt{direct} / \texttt{inferred} / \texttt{derivable}
support tag, which is the source of the per-payload weights
\(r(p)\in\{1.00, 0.65, 0.45, 0\}\) used by \hardok{} in
Appendix~\ref{app:hardok-audit}.

\begin{table*}[!t]
\centering
\footnotesize
\setlength{\tabcolsep}{4pt}
\begin{tabular}{lp{11.0cm}r}
\toprule
\textbf{Layer} & \textbf{Members} & \textbf{$|\cdot|$} \\
\midrule
Canonical actions &
\texttt{MoveTo}, \texttt{Interact}, \texttt{Speak}, \texttt{Trade},
\texttt{Craft}, \texttt{Deliver}, \texttt{Patrol}, \texttt{CheckIn},
\texttt{CheckOut}, \texttt{Harvest}, \texttt{SignContract}, \texttt{PublishTask}
& 12 \\
Observables &
\texttt{position}, \texttt{room\_presence}, \texttt{object\_possession},
\texttt{object\_state}, \texttt{need\_state\_delta},
\texttt{resource\_affordance}, \texttt{social\_interaction\_event}
& 7 \\
Evidence channels &
\texttt{movement\_log}, \texttt{room\_presence\_log}, \texttt{state\_log},
\texttt{object\_interaction\_log}, \texttt{social\_graph\_log},
\texttt{task\_log}, \texttt{contract\_log}, \texttt{owner\_log},
\texttt{checkin\_log}, \texttt{checkout\_log}, \texttt{patrol\_report}
& 11 \\
Room families &
\texttt{delivery\_station}, \texttt{learning\_and\_public\_activity},
\texttt{food\_service}, \texttt{private\_living}, \texttt{storage\_and\_utility},
\texttt{movement\_boundary}, \texttt{health\_and\_sport}
& 7 \\
Resource families &
\texttt{consumable\_food}, \texttt{hydration\_source}, \texttt{rest\_resource},
\texttt{play\_learning\_resource}, \texttt{exchange\_delivery\_object},
\texttt{boundary\_access\_object}, \texttt{tool\_equipment},
\texttt{hygiene\_health\_resource}
& 8 \\
\bottomrule
\end{tabular}
\caption{Operational primitives exposed to generators. The vocabulary is fixed
across all methods and seeds; the search space \(\mathcal{A}(x)\) of
Section~\ref{sec:method} is built from these primitives plus the goal-level
templates.}
\label{tab:bench-record-layer}
\end{table*}

\subsection{RealCharter-Bench Document Lineage}

RealCharter-Bench contains 53 tasks drawn from 15 publicly available policy
documents and grouped into four source families.
Table~\ref{tab:bench-realcharter} lists the source documents and the per-family
task counts. All documents are public Code-of-Conduct pages, governance
charters, or handbook sections; the Contributor Covenant template is
distributed under CC~BY~4.0 and the remaining sources are documented with
retrieval dates in the released benchmark materials.

\begin{table*}[!t]
\centering
\footnotesize
\setlength{\tabcolsep}{4pt}
\begin{tabular}{lp{8.0cm}rr}
\toprule
\textbf{Source family / domain} & \textbf{Public source documents} &
\textbf{Docs} & \textbf{Tasks} \\
\midrule
oss\_governance / software &
Python PSF, Django, Apache, Drupal, Node.js, CNCF, Debian, Contributor Covenant
& 8 & 25 \\
community\_rules / community &
Stack Overflow, Wikimedia UCC, Ubuntu, Mozilla
& 4 & 16 \\
event\_platform\_safety / event &
Linux Foundation Events, NumFOCUS
& 2 & 8 \\
public\_handbook / workplace &
GitLab Handbook
& 1 & 4 \\
\midrule
\textbf{Total} & 15 public policy sources & \textbf{15} & \textbf{53} \\
\bottomrule
\end{tabular}
\caption{Provenance of the 53 RealCharter-Bench tasks. Each task is exposed to
generators as a normalized policy card containing source family, domain,
context, roles, trigger conditions, normative summary, required and forbidden
actions, evidence channels, procedure requirements, and expected slots. The
original clause text is withheld from generators.}
\label{tab:bench-realcharter}
\end{table*}

\section{Extended Results}
\label{app:result-suite}

This appendix consolidates the supporting result tables behind the headline
claims. Subsection~\ref{app:metric-sensitivity} reports how the disclosed
\overall{} composite behaves as the \hardok{} weight is varied;
Subsection~\ref{app:hardok-audit} explains the deliberately uncomfortable
\hardok{} gap between \ours{} and the No LLM candidates ablation;
Subsection~\ref{app:crossmodel} records the completed cross-model run status
and the omitted generator columns; and
Subsection~\ref{app:full-results} expands the main, OOD, and RealCharter
tables with bootstrap intervals, candidate-pool diagnostics, and
human-calibration breakdowns. Table~\ref{tab:result-contrasts-app} is the
single-page map of the contrasts that matter most across the result suite;
positive gaps mean the first method or condition is better on that metric,
and negative gaps are retained where they reveal a real trade-off.

\begin{table*}[!t]
\centering
\footnotesize
\setlength{\tabcolsep}{4pt}
\begin{tabular}{lllrrr}
\toprule
\textbf{Question} & \textbf{Contrast} & \textbf{Metric} &
\textbf{Primary} & \textbf{Comparator} & \textbf{Gap} \\
\midrule
Main benchmark & \ours{} vs. No LLM candidates & \overall{} & 74.99 & 69.25 & +5.74 \\
Main benchmark & \ours{} vs. No LLM candidates & Avg & 80.95 & 67.24 & +13.71 \\
Main benchmark & \ours{} vs. No LLM candidates & \hardok{} & 51.14 & 77.27 & -26.14 \\
Main benchmark & \ours{} vs. Direct Prompting & readability & 79.47 & 79.77 & -0.30 \\
Ablation study & \ours{} vs. random candidates & Avg & 80.95 & 79.92 & +1.03 \\
Ablation study & \ours{} vs. uniform weights & \hardok{} & 51.14 & 60.61 & -9.47 \\
Ablation study & \ours{} vs. CP-SAT control & \overall{} & 74.99 & 63.94 & +11.05 \\
Slot intervention & baseline vs. degraded \slot{NormBody} & Avg & 49.65 & 44.98 & +4.67 \\
Slot intervention & baseline vs. degraded \slot{EvidencePolicy} & Avg & 49.65 & 48.47 & +1.18 \\
Held-out OOD & \ours{} vs. No LLM candidates & \overall{} & 82.71 & 68.71 & +14.00 \\
GPT-5.1 robustness & \ours{} vs. Direct Prompting & Avg & 81.53 & 84.08 & -2.55 \\
GPT-5.1 robustness & \ours{} vs. No LLM candidates & \overall{} & 81.97 & 69.10 & +12.86 \\
RealCharter-Bench & \ours{} vs. Self-Refine & \overall{} & 94.69 & 86.09 & +8.60 \\
Human calibration & \ours{} vs. No search & human mean & 4.20 & 3.51 & +0.69 \\
Second evaluator & \ours{} vs. best non-\ours{} & Gemini Avg & 96.40 & 85.60 & +10.80 \\
\bottomrule
\end{tabular}
\caption{Selected contrasts for interpreting the result suite. Values are drawn
from the reported results and are included to expose the trade-offs behind
the headline claims rather than to introduce a new ranking.}
\label{tab:result-contrasts-app}
\end{table*}

Read together, the contrasts in Table~\ref{tab:result-contrasts-app} say that
\ours{} consistently wins on the disclosed \overall{} composite while
deliberately losing on the standalone binary \hardok{} gate (because richer
selected payloads stress a ratio-based feasibility threshold), that the
slot-sensitivity gap is concentrated in the rule's semantic core
(\slot{NormBody}, \slot{Scope}, \slot{Trigger}) rather than in evidence or
procedure, and that the CP-SAT control isolates the gain to structured search
over the candidate inventory rather than to discrete optimization in general.
The remaining subsections elaborate each of these contrasts.

\subsection{Metric Sensitivity}
\label{app:metric-sensitivity}

The executable \overall{} score uses a 20\% \hardok{} weight. We audited
alternative weights from 0\% to 40\%. \ours{} ranks first in the main benchmark,
every cross-model robustness column, and held-out OOD evaluation for 10\%, 20\%,
and 30\% weights. At 40\%, No LLM candidates overtakes \ours{} in the
GPT-4o-mini main setting, indicating that 40\% is too feasibility-heavy for a
headline metric.

This analysis addresses a potential concern about the headline metric: if the
\hardok{} term were tuned only to make \ours{} win, the ranking should be unstable
under small changes in its weight. Table~\ref{tab:metric-sensitivity-app}
shows the opposite for the reported 20\% setting. \ours{} is already strongest on pure
Avg in the main and OOD settings, and it remains strongest for 10--30\%
\hardok{} weights. In the cross-model setting, stronger generators sometimes
make Direct Prompting competitive on pure Avg, but the moment executable
feasibility is given even a small weight, \ours{} becomes the top method. The
only warning sign is the 40\% main setting, where No LLM candidates'
conservative feasibility dominates despite lower rubric quality; this is why the
paper treats 40\% as too feasibility-heavy.

\begin{table*}[!t]
\centering
\small
\begin{tabular}{llrrrr}
\toprule
\textbf{Context} & \(\boldsymbol{\alpha}\) & \textbf{Winner} &
\textbf{\ours{} rank} & \textbf{\ours{} score} & \textbf{Margin} \\
\midrule
Main benchmark & 0.0 & \ours{} & 1 & 80.95 & +12.12 \\
Main benchmark & 0.1 & \ours{} & 1 & 77.97 & +9.73 \\
Main benchmark & 0.2 & \ours{} & 1 & 74.99 & +5.74 \\
Main benchmark & 0.3 & \ours{} & 1 & 72.01 & +1.76 \\
Main benchmark & 0.4 & No LLM candidates & 2 & 69.03 & -2.23 \\
Held-out OOD & 0.0 & \ours{} & 1 & 79.78 & +14.22 \\
Held-out OOD & 0.1 & \ours{} & 1 & 81.24 & +16.44 \\
Held-out OOD & 0.2 & \ours{} & 1 & 82.71 & +14.00 \\
Held-out OOD & 0.3 & \ours{} & 1 & 84.18 & +11.56 \\
Held-out OOD & 0.4 & \ours{} & 1 & 85.64 & +9.11 \\
GPT-5.1 robustness & 0.0 & Direct Prompting & 2 & 81.53 & -2.55 \\
GPT-5.1 robustness & 0.2 & \ours{} & 1 & 81.97 & +12.86 \\
Gemini robustness & 0.0 & Direct Prompting & 2 & 73.15 & -3.91 \\
Gemini robustness & 0.2 & \ours{} & 1 & 73.75 & +4.71 \\
Qwen3 robustness & 0.0 & Direct Prompting & 2 & 74.89 & -3.98 \\
Qwen3 robustness & 0.2 & \ours{} & 1 & 74.08 & +4.91 \\
\bottomrule
\end{tabular}
\caption{Sensitivity of the executable composite
\(\overall_\alpha=(1-\alpha)\mathrm{Avg}+\alpha\hardok{}\). Margins are
\ours{} minus the best non-\ours{} method under the same context and weight.}
\label{tab:metric-sensitivity-app}
\end{table*}

The main takeaway is that the paper's conclusion is not an artifact of a single
hidden scalarization. The disclosed 20\% setting is in the stable region where
\ours{} wins across the main benchmark, held-out OOD evaluation, and all
completed cross-model columns. At the same time, exposing the 40\% failure case
is useful: it makes clear that \hardok{} is a diagnostic feasibility term rather
than a complete definition of rule quality.

\subsection{Cross-model Run Status}
\label{app:crossmodel}

The cross-model experiment was designed to compare Direct Prompting,
Self-Refine, No LLM candidates, and \ours{} across several generator models with
a fixed GPT-5.1 scorer. The completed run set has
4224/4224 expected runs. GPT-4o and GLM-4.6 were omitted from the final report
after partial-run and timeout issues; Gemini Flash Lite and Qwen3-235B provide
the non-GPT generator columns.

The purpose of this appendix note is to separate completed evidence from planned
coverage. The reported cross-model claim should be read as: among the completed
generator columns, and under the same independent GPT-5.1 scorer, \ours{} has the
highest executable \overall{} score in every column. It should not be read as a
claim about every commercial or open model family. The omitted GPT-4o and
GLM-4.6 columns were excluded because their runs were not complete enough to
support a balanced method comparison, not because they changed the reported
ranking.

\subsection{Hard-OK Forensic Audit}
\label{app:hardok-audit}

The empirical claims rely on the disclosed executable \overall{} metric, so we
audit its most fragile ingredient. The main table contains a deliberately
uncomfortable result: the No LLM candidates ablation has higher \hardok{} than
full \ours{}. A forensic audit found that the stored \hardok{} value is computed
at the AOG level rather than from the polished surface text, so the drop is not
primarily caused by the polish stage deleting grounded tokens. Instead, the full
\ours{} path keeps an average of 16.65 LLM-augmented candidates per run and
selects slightly more payloads, including more non-goal evidence payloads. This
increases expressiveness and rubric quality, but makes the ratio-based
real-data-grounding gate more fragile. Averaged over the main runs, \ours{} has
\(\rho=0.507\), 6.19 directly supported payloads, 6.08 inferred payloads, and
6.70 missing payloads; No LLM candidates is more conservative with
\(\rho=0.551\) and 6.30 missing payloads. A representative service-shift run
illustrates the gate: both methods cover CheckIn/CheckOut actions and movement
or check-in logs, but \ours{} also selects additional procedure/evidence
payloads; with 4 direct, 7 inferred, and 6 missing supported items,
\(\rho=0.479<0.50\), so \hardok{} fails despite a more specific clause. We
therefore keep \hardok{} visible and use it only as the disclosed 20\%
feasibility term in \overall{}, rather than folding it silently into Avg.

\subsection{Additional Result Details}
\label{app:full-results}

This appendix reports additional result details as ordinary paper tables rather
than raw per-example listings.

Table~\ref{tab:bootstrap-app} reports paired bootstrap intervals for the main
comparisons. The resampling unit is the goal: for the main benchmark and
held-out OOD evaluation, seed-level scores
are first averaged within each goal, and bootstrap samples are drawn over goals.
This makes the interval a statement about task variation rather than about
random seed noise alone. The margin column is always \ours{} minus the strongest
available non-\ours{} method for the same experiment and metric.

\begin{table*}[!t]
\centering
\small
\begin{tabular}{llrrrrr}
\toprule
\textbf{Experiment} & \textbf{Metric} & \textbf{Best non-\ours{}} &
\textbf{n} & \textbf{\ours{}} & \textbf{Margin} & \textbf{95\% CI} \\
\midrule
Main benchmark & \overall{} & No LLM candidates & 88 & 74.99 & +5.74 & [3.93, 7.53] \\
Main benchmark & Avg & No LLM candidates & 88 & 80.95 & +13.71 & [12.65, 14.85] \\
Main benchmark & \hardok{} & No LLM candidates & 88 & 51.14 & -26.14 & [-34.47, -18.56] \\
Main benchmark & Faith & No LLM candidates & 88 & 3.96 & +0.56 & [0.42, 0.70] \\
Held-out OOD & \overall{} & No LLM candidates & 12 & 82.71 & +14.00 & [11.42, 16.49] \\
Held-out OOD & Avg & No LLM candidates & 12 & 79.78 & +18.89 & [16.11, 21.67] \\
Held-out OOD & \hardok{} & No LLM candidates & 12 & 94.44 & -5.56 & [-13.89, 0.00] \\
Held-out OOD & Faith & No LLM candidates & 12 & 3.92 & +1.11 & [0.94, 1.28] \\
RealCharter & \overall{} & Self-Refine & 53 & 94.69 & +8.60 & [5.25, 12.22] \\
RealCharter & Avg & Self-Refine & 53 & 93.36 & +5.09 & [3.61, 6.64] \\
RealCharter & \hardok{} & Self-Refine & 53 & 100.00 & +22.64 & [11.32, 33.96] \\
RealCharter & Faith & Self-Refine & 53 & 5.00 & +0.23 & [0.13, 0.35] \\
\bottomrule
\end{tabular}
\caption{Paired bootstrap confidence intervals for \ours{} against the best
non-\ours{} method in each experiment. For the main and held-out evaluations,
seeds are averaged within goal before resampling.}
\label{tab:bootstrap-app}
\end{table*}

The bootstrap results sharpen the main narrative. On the main benchmark,
\ours{}'s \overall{} margin over No LLM candidates is positive with a 95\%
interval of [3.93, 7.53], and the Avg margin is larger still. At the same time,
the \hardok{} row shows the visible trade-off: No LLM candidates has higher
binary feasibility coverage in the main setting, while \ours{} wins the combined
task-validity score because its rubric quality is substantially higher. The
held-out OOD and RealCharter intervals are also positive on \overall{},
supporting the paper's robustness and external-transfer claims while preserving
the distinction between quality, feasibility, and faithfulness.

Table~\ref{tab:realcharter-full-app} gives the full RealCharter-Bench breakdown.
This benchmark is constructed from public charter and rule clauses, but the
source clauses are hidden from generators. All methods receive the same
normalized task cards. Self-Refine is included as the strongest LLM workflow baseline,
and the no-search ablation is included to test whether the gain comes merely
from the \ours{} renderer and inventory rather than from structured candidate
scoring.

\begin{table*}[!t]
\centering
\small
\begin{tabular}{lrrrrrrrr}
\toprule
\textbf{Method} & \textbf{n} & \textbf{Spec} & \textbf{Exec} &
\textbf{Flu} & \textbf{Read} & \textbf{Faith} & \textbf{\hardok{}} &
\textbf{\overall{}} \\
\midrule
Direct Prompting & 53 & 4.60 & 3.87 & 4.44 & 4.27 & 4.64 & 62.26 & 82.26 \\
Self-Refine & 53 & 4.77 & 4.06 & 4.28 & 4.20 & 4.77 & 77.36 & 86.09 \\
No LLM candidates & 53 & 4.52 & 3.58 & 4.22 & 4.00 & 4.50 & 47.17 & 76.07 \\
No search & 53 & 2.39 & 1.73 & 4.50 & 4.50 & 2.68 & 0.00 & 50.56 \\
\ours{} & 53 & 5.00 & 4.34 & 4.50 & 4.50 & 5.00 & 100.00 & 94.69 \\
\bottomrule
\end{tabular}
\caption{Full RealCharter-Bench metric breakdown. Direct Prompting and
Self-Refine are LLM workflow baselines over the same normalized task cards; No
search uses the \ours{} inventory and renderer without candidate scoring or
search.}
\label{tab:realcharter-full-app}
\end{table*}

The RealCharter result is therefore stronger than a comparison against Direct
Prompting and No LLM candidates alone. Self-Refine narrows the gap by adding a
draft-and-refine workflow, but \ours{} still leads on faithfulness, \hardok{},
and \overall{}. The no-search ablation fails on faithfulness and feasibility
despite using the same broad rule inventory, indicating that the search step is
doing real work. The appropriate claim is bounded: \ours{} generalizes to a
real-derived external benchmark under this
task-card protocol; the experiment does not claim that generated rules are
better than the original human-authored clauses.

Table~\ref{tab:candidate-source-app} first audits everyday candidate
construction and selected payload sources in the main test runs.
Table~\ref{tab:candidate-pool-app} then audits the candidate pools used by the
slot-sensitivity intervention. The key diagnostic there is replacement coverage.
If a slot ablation failed to find replacement candidates, then a small score drop
could be explained by a weak intervention rather than by genuine slot robustness.
Full coverage across all five slots rules out that trivial explanation.

\begin{table*}[!t]
\centering
\small
\setlength{\tabcolsep}{4pt}
\begin{tabular}{lrrrr}
\toprule
\textbf{Audit quantity} & \textbf{Template/goal} & \textbf{Schema} &
\textbf{Record} & \textbf{LLM} \\
\midrule
Candidates generated per test goal & 14.0 & 9.3 & 10.2 & 16.7 \\
LLM candidates dropped per run & -- & -- & -- & 1.3 \\
Selected payloads/run, \ours{} & 9.3 & 3.0 & 8.3 & 3.5 \\
Selected payloads/run, No LLM candidates & 8.4 & 2.9 & 12.0 & -- \\
\bottomrule
\end{tabular}
\caption{Candidate-source audit for main-test search. Deterministic candidate
counts are reconstructed from the frozen test goals and context/record
inventory; LLM kept/dropped counts and selected-payload sources are computed from
completed main-run artifacts.}
\label{tab:candidate-source-app}
\end{table*}

\begin{table}[!t]
\centering
\footnotesize
\setlength{\tabcolsep}{4pt}
\begin{tabular}{lrrrr}
\toprule
\textbf{Slot} & \textbf{Cand.} & \textbf{Distr.} &
\textbf{Sig.} & \textbf{Cov.} \\
\midrule
\slot{Scope} & 7.00 & 812 & 52 & 100.0 \\
\slot{Trigger} & 4.23 & 464 & 71 & 100.0 \\
\slot{NormBody} & 6.52 & 754 & 103 & 100.0 \\
\slot{EvidencePolicy} & 6.22 & 697 & 46 & 100.0 \\
\slot{Procedure} & 9.15 & 1015 & 16 & 100.0 \\
\bottomrule
\end{tabular}
\caption{Candidate-pool audit for the slot-sensitivity experiment. Replacement
coverage is the percentage of slot-intervention runs in which the degraded slot
was actually replaced.}
\label{tab:candidate-pool-app}
\end{table}

The audit also helps interpret the surprising evidence-policy result. Evidence
has full replacement coverage and nontrivial distractor mass, so its smaller
sensitivity is not simply an artifact of missing degraded evidence candidates.
The more plausible interpretation is structural: in this benchmark, errors in
\slot{NormBody}, \slot{Scope}, and \slot{Trigger} more directly change who must
do what and when, while several evidence-policy variants remain compatible with
the same underlying rule. This is why the paper presents the slot analysis as an
empirical diagnostic rather than as a universal theory of normative language.

Table~\ref{tab:human-calibration-app} reports the method-blinded human
calibration study for the main benchmark sample. It is included as a
human-validity check on the model-based rubric, not as an additional automatic
leaderboard.

\paragraph{Human annotation protocol.}
The calibration used three volunteer PhD annotators from different disciplinary
backgrounds. They were recruited as academic volunteers, received no payment or
other compensation, and could decline or stop without penalty. Before
annotation, they were told that their scores would be used only in aggregate for
a research paper evaluating automatic rule-generation methods; no personally
identifying information or sensitive personal data were collected. Because the
task consisted only of rating generated text against synthetic or normalized task
cards, and did not collect private data, involve intervention, or evaluate the
annotators themselves, the protocol was not submitted for IRB or ethics-board
review. No screenshots or interactive interface were part of the instructions.

The written instruction given to annotators was: ``For each row, read only the
task card and generated clause. Method names, automatic scores, source clauses,
and search traces are hidden. Do not reward a clause merely for being long or
formal. A good clause should be faithful to the task, grounded in the provided
roles/actions/evidence/context, procedurally auditable, and readable as an
organizational rule. Score intent faithfulness, grounding, procedural closure,
and readability on a 1--5 ordinal scale. Use the full 1--5 range: if a clause is
fluent but invents unavailable evidence or procedures, lower grounding; if a
clause is grounded but terse, grounding can remain high while readability or
procedural closure can be lower.''

\begin{table*}[!t]
\centering
\small
\begin{tabular}{lrrrrrr}
\toprule
\textbf{Method} & \textbf{n} & \textbf{Faith} &
\textbf{Ground} & \textbf{Proc.} & \textbf{Read} & \textbf{Mean} \\
\midrule
Direct Prompting & 20 & 3.93 & 2.22 & 2.98 & 4.62 & 3.44 \\
Self-Refine & 20 & 3.72 & 3.43 & 2.82 & 2.52 & 3.12 \\
No search & 20 & 3.87 & \best{3.48} & 3.47 & 3.22 & 3.51 \\
No LLM candidates & 20 & 3.87 & 2.32 & 3.22 & 3.55 & 3.24 \\
\ours{} & 20 & \best{4.53} & 3.42 & \best{4.12} & \best{4.72} & \best{4.20} \\
\bottomrule
\end{tabular}
\caption{Human calibration details for the main benchmark sample. Three
annotators scored 100 method-blinded outputs from 20 goals. Faith is intent
faithfulness, Ground is grounding, Proc. is procedural closure, and Mean averages
the four 1--5 dimensions.}
\label{tab:human-calibration-app}
\end{table*}

Table~\ref{tab:human-calibration-app} gives the full method breakdown for the
human calibration sample summarized in Section~\ref{sec:main-results}. The sampled
rows contain no method labels, automatic scores, source clauses, or search traces;
method labels are joined only after annotation. The strongest human signal is not
that \ours{} maximizes every
subdimension: No search slightly leads on grounding. Rather, \ours{} combines
high intent faithfulness, procedural closure, and readability, producing the
highest mean score overall. This is consistent with the main evaluator's
ranking while preserving the important \hardok{} caveat discussed in
Appendix~\ref{app:hardok-audit}.

\begin{table}[!t]
\centering
\footnotesize
\setlength{\tabcolsep}{4pt}
\begin{tabular}{lrr}
\toprule
\textbf{Agreement quantity} & \textbf{Value} & \textbf{Scope} \\
\midrule
Human--GPT Spearman \(\rho\) & 0.77 & row-level mean \\
Human--GPT Kendall \(\tau\) & 0.62 & row-level mean \\
Human--GPT Pearson \(r\) & 0.91 & method means \\
Pairwise human Spearman & 0.65 & annotator means \\
Ordinal within-one agreement & 85.6 & all 1--5 labels \\
\bottomrule
\end{tabular}
\caption{Agreement statistics for the human calibration study. Ordinal
within-one agreement is the average pairwise percentage of labels within one
point across the four 1--5 dimensions.}
\label{tab:human-agreement-app}
\end{table}

Table~\ref{tab:second-evaluator-app} reports a separate second-model evaluator
audit on the same 100-row calibration pack. This is not a human study and not a
replacement for the main GPT-5.1 rubric. It asks a narrower question: if a
different judge family scores the same blinded outputs with the same five
1--5 rubric dimensions, does the main conclusion disappear? The answer is no.
The Gemini-family judge has moderate item-level agreement with GPT-5.1 Avg
(Pearson \(r=0.575\), Spearman \(\rho=0.611\), Kendall \(\tau_b=0.480\)).
It does not reproduce the full ordering of the non-\ours{} methods, but it
independently ranks \ours{} highest by a large margin.

\begin{table}[!t]
\centering
\footnotesize
\setlength{\tabcolsep}{4pt}
\resizebox{\columnwidth}{!}{%
\begin{tabular}{lrrrr}
\toprule
\textbf{Method} & \textbf{n} & \textbf{GPT Avg} &
\textbf{Gemini Avg} & \textbf{Gemini \overall{}} \\
\midrule
Direct Prompting & 20 & 68.80 & 71.40 & 57.12 \\
Self-Refine & 20 & 58.20 & 78.60 & 63.88 \\
No search & 20 & 64.80 & 85.60 & 75.48 \\
No LLM candidates & 20 & 67.80 & 84.80 & 82.84 \\
\ours{} & 20 & \best{80.40} & \best{96.40} & \best{86.12} \\
\bottomrule
\end{tabular}}
\caption{Second-model evaluator sanity check on the 100-row human-calibration
pack. GPT Avg is the primary GPT-5.1 v1 rubric mean. Gemini Avg is scored by a
Gemini-family judge with the same five rubric dimensions.}
\label{tab:second-evaluator-app}
\end{table}

Table~\ref{tab:failure-modes-app} summarizes the recurring failure modes that
the quantitative analyses point to. The table is diagnostic rather than a new
experiment: it connects the main score table, the slot intervention, the
candidate-pool audit, and RealCharter-Bench into the same interpretation.

\begin{table*}[!t]
\centering
\footnotesize
\setlength{\tabcolsep}{4pt}
\begin{tabularx}{\textwidth}{@{}p{0.18\textwidth}YYY@{}}
\toprule
\textbf{Failure mode} & \textbf{Typical symptom} & \textbf{Evidence in results} &
\textbf{Implication} \\
\midrule
Surface-only fluency & The clause reads well but cites generic systems,
manager discretion, or unsupported testimony. & Direct Prompting has strong
fluency/readability but zero main-table \hardok{} and lower faithfulness. &
Language quality must be separated from operational validity. \\
Syntactic structure without grounding & The output follows a schema or
workflow but fails to bind the fields to available records. & Structured
Outputs, Best-of-\(N\), and Self-Refine remain weak on executability and
\hardok{}. & A parseable intermediate form is not enough unless candidates are
record-grounded. \\
Conservative under-generation & The system stays close to safe record payloads
but loses specificity, executability, or naturalness. & No LLM candidates has
the highest main-table \hardok{} but much lower Avg and \overall{} than
\ours{}. & Feasibility coverage should remain visible rather than replace the
rubric score. \\
Semantic-core corruption & The rule changes who is covered, when it activates,
or what conduct is regulated. & Slot degradation hurts most for
\slot{NormBody}, \slot{Scope}, and \slot{Trigger}. & Search should prioritize
the semantic core before optimizing evidence wording or procedure. \\
Renderer-only transfer & A real-derived task card is formatted into prose
without selecting the right operational structure. & The RealCharter no-search
ablation collapses on faithfulness and feasibility. & External transfer depends
on structured search, not just on templates over policy fields. \\
\bottomrule
\end{tabularx}
\caption{Diagnostic failure modes implied by the result suite. The table
collects evidence across the reported experiments rather than introducing an
additional scoring procedure.}
\label{tab:failure-modes-app}
\end{table*}

\section{Qualitative Examples}
\label{app:qualitative}

This appendix gives compact qualitative examples for the main diagnostic claims.
The examples are selected from completed test runs and are not additional
evaluation data. We paraphrase only enough context to make the failure readable;
method labels are the paper-facing names used in the main tables.

Figure~\ref{fig:aog-case-study} visualizes one service-shift example as an AOG,
and Table~\ref{tab:qual-aog-trace} unpacks the corresponding search trace. The
trace is shortened for readability, but it shows the level at which \ours{}
searches: not over whole clauses, but over slot-level operational commitments
that can be checked before prose rendering.

\begin{figure*}[!t]
\centering
\resizebox{\textwidth}{!}{%
\begin{tikzpicture}[
  font=\scriptsize,
  title/.style={align=center, font=\bfseries\scriptsize},
  root/.style={draw, rounded corners, align=center, text width=3.25cm,
    minimum height=0.72cm, fill=gray!10},
  andnode/.style={draw, circle, align=center, minimum size=0.78cm,
    fill=orange!12, font=\bfseries\scriptsize},
  ornode/.style={draw, diamond, aspect=1.7, align=center, inner sep=1pt,
    fill=blue!8, font=\bfseries\scriptsize},
  candnode/.style={draw, rounded corners, align=center, text width=1.35cm,
    minimum height=0.64cm, fill=green!8},
  altnode/.style={draw, dashed, rounded corners, align=center, text width=1.35cm,
    minimum height=0.58cm, fill=red!5},
  keepnode/.style={draw, rounded corners, align=center, text width=1.35cm,
    minimum height=0.58cm, fill=gray!8},
  stepnode/.style={draw, circle, inner sep=1.3pt, fill=green!15},
  rejnode/.style={draw, circle, inner sep=1.3pt, fill=red!18},
  legend/.style={draw, rounded corners, align=left, text width=3.05cm,
    fill=gray!6, inner sep=4pt},
  edge/.style={-{Latex[length=2mm]}, thick, shorten >=1pt, shorten <=1pt},
  weakedge/.style={-{Latex[length=1.8mm]}, dashed, thick, shorten >=1pt, shorten <=1pt},
  scoreline/.style={-{Latex[length=2mm]}, thick, blue!70!black}
]
\node[title] at (-3.9,5.15) {(a) And-Or structure};
\node[title] at (4.75,5.15) {(b) MCMC proposal trajectory};

\node[root] (goal) at (-3.9,4.35) {Goal \(g\): coordinate service shifts\\
with auditable arrival and closure};
\node[andnode] (and) at (-3.9,3.35) {AND};

\node[ornode] (scope) at (-7.55,2.0) {OR\\scope};
\node[ornode] (trigger) at (-5.72,2.0) {OR\\trigger};
\node[ornode] (norm) at (-3.90,2.0) {OR\\norm};
\node[ornode] (evidence) at (-2.08,2.0) {OR\\evidence};
\node[ornode] (procedure) at (-0.25,2.0) {OR\\procedure};

\node[candnode] (s1) at (-7.55,0.95) {S1 staff\\roles};
\node[altnode] (s2) at (-7.55,0.05) {S2 generic\\members};
\node[keepnode] (s3) at (-7.55,-0.85) {S3 patrol\\roles};

\node[candnode] (t1) at (-5.72,0.95) {T1 check-in\\\(\leq15\) min};
\node[altnode] (t2) at (-5.72,0.05) {T2 unbounded\\start};
\node[keepnode] (t3) at (-5.72,-0.85) {T3 speak\\event};

\node[candnode] (n1) at (-3.90,0.95) {N1 require\\check-in/out};
\node[keepnode] (n2) at (-3.90,0.05) {N2 require\\speak};
\node[keepnode] (n3) at (-3.90,-0.85) {N3 forbid\\trade};

\node[candnode] (e1) at (-2.08,0.95) {E1 movement +\\check logs};
\node[altnode] (e2) at (-2.08,0.05) {E2 attendance\\timestamp};
\node[keepnode] (e3) at (-2.08,-0.85) {E3 social\\graph};

\node[candnode] (p1) at (-0.25,0.95) {P1 manager\\review};
\node[keepnode] (p2) at (-0.25,0.05) {P2 no\\appeal};
\node[keepnode] (p3) at (-0.25,-0.85) {P3 informal\\reminder};

\draw[edge] (goal) -- (and);
\foreach \x in {scope,trigger,norm,evidence,procedure} {
  \draw[edge] (and) -- (\x);
}
\foreach \slot/\a/\b/\c in {scope/s1/s2/s3,trigger/t1/t2/t3,norm/n1/n2/n3,evidence/e1/e2/e3,procedure/p1/p2/p3} {
  \draw[edge] (\slot.south) -- (\a.north);
  \draw[weakedge] (\slot.south west) .. controls +(-0.42,-0.45) and +(-0.58,0.35) .. (\b.west);
  \draw[weakedge] (\slot.south east) .. controls +(0.42,-0.62) and +(0.58,0.42) .. (\c.east);
}

\node[legend, text width=2.35cm] at (-7.35,-1.95) {solid green = selected\\
dashed red = rejected\\gray = unselected};

\draw[thick] (2.05,-1.25) -- (2.05,3.95);
\draw[thick] (2.05,-1.25) -- (5.75,-1.25);
\node[rotate=90] at (1.64,1.35) {internal score \(\ssearch\)};
\node at (3.95,-1.72) {proposal step};
\node[anchor=east] at (2.0,-1.25) {low};
\node[anchor=east] at (2.0,3.9) {high};

\coordinate (p0) at (2.35,-0.82);
\coordinate (p1) at (2.95,-0.18);
\coordinate (p2) at (3.48,0.72);
\coordinate (p3) at (4.05,1.42);
\coordinate (p4) at (4.62,2.20);
\coordinate (p5) at (5.25,2.95);
\draw[scoreline] (p0) -- (p1) -- (p2) -- (p3) -- (p4) -- (p5);
\foreach \p in {p0,p1,p2,p3,p4,p5} {
  \node[stepnode] at (\p) {};
}

\node[anchor=north] at (2.35,-0.95) {\(a_0\)};
\node[anchor=south] at (2.95,-0.08) {\(a_1\)};
\node[anchor=south] at (3.48,0.82) {\(a_2\)};
\node[anchor=south] at (4.05,1.52) {\(a_3\)};
\node[anchor=south] at (4.62,2.30) {\(a_4\)};
\node[anchor=south west] at (5.25,2.95) {\(a^*\)};

\coordinate (r1) at (3.48,-0.05);
\coordinate (r2) at (4.05,0.55);
\node[rejnode] at (r1) {};
\node[rejnode] at (r2) {};
\draw[weakedge] (p2) -- (r1);
\draw[weakedge] (p3) -- (r2);
\node[anchor=east] at (r1) {R1};
\node[anchor=east] at (r2) {R2};

\node[legend] at (7.22,1.55) {\textbf{Accepted}\\
1: S2 \(\rightarrow\) S1 scope\\
2: T2 \(\rightarrow\) T1 trigger\\
3: E2 \(\rightarrow\) E1 evidence\\
4: P2 \(\rightarrow\) P1 procedure\\[2pt]
\textbf{Rejected}\\
R1: unsupported timestamp\\
R2: scope drift};
\end{tikzpicture}}
\caption{AOG schematic and MCMC search trace for the service-shift case study.
Left: an AND node requires a complete five-slot rule, while each OR node selects
one candidate from a slot-local pool containing selected, rejected, and unused
alternatives. Right: a shortened proposal trajectory shows accepted edits raising
the internal structural score and rejected edits falling off the path, with
proposal details moved to the legend to keep the curve readable. Vertical
position is qualitative and illustrates the search dynamics rather than adding a
new reported metric.}
\label{fig:aog-case-study}
\end{figure*}
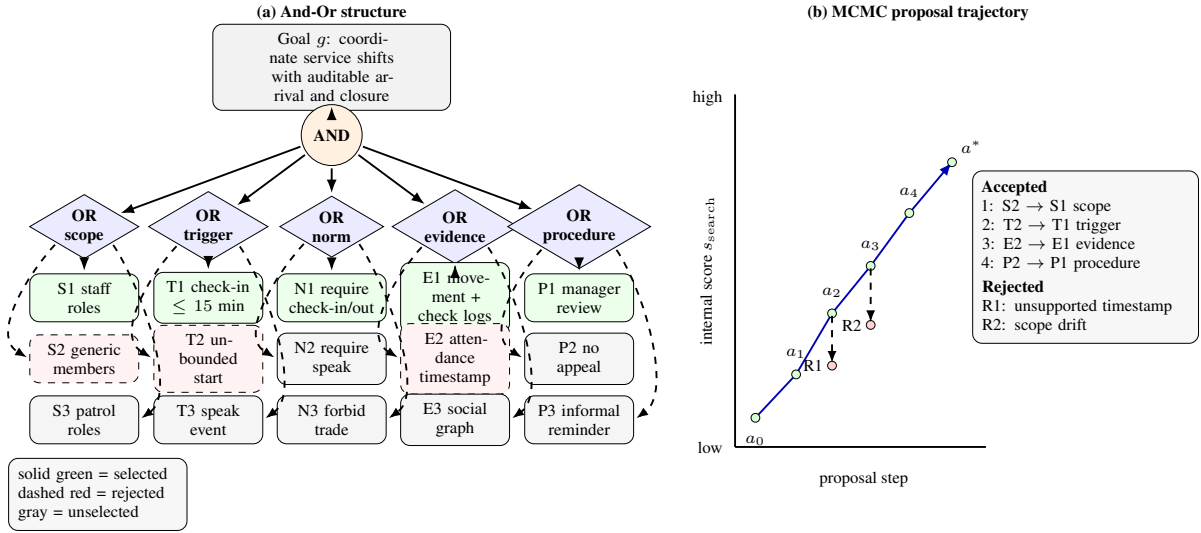

\begin{table*}[!t]
\centering
\footnotesize
\setlength{\tabcolsep}{4pt}
\begin{tabularx}{\textwidth}{@{}p{0.15\textwidth}YY@{}}
\toprule
\textbf{Stage} & \textbf{Intermediate structure} & \textbf{Role in the search} \\
\midrule
Goal and records &
Goal: coordinate shared service shifts with auditable arrival and closure.
Relevant vocabulary includes restaurant roles, \texttt{CheckIn},
\texttt{CheckOut}, \texttt{movement\_log}, \texttt{checkin\_log}, and
\texttt{checkout\_log}. &
The goal supplies the normative intent; the record layer restricts which actors,
actions, and evidence channels can be used without inventing unavailable logs. \\
Candidate pools &
\slot{Scope}: restaurant staff and shift records; \slot{Trigger}: arrival and
closure events; \slot{NormBody}: required \texttt{CheckIn}/\texttt{CheckOut};
\slot{EvidencePolicy}: movement/check-in/check-out logs; \slot{Procedure}:
manager review and correction path. &
Each slot receives multiple candidates from templates, schema terms, record
traces, and validated LLM proposals. The AOG search space is the cross-product of
these local choices. \\
Initial AOG state &
A safe but weak state covers generic members and requires
\texttt{CheckIn}/\texttt{CheckOut}, with broad evidence wording and little
procedural closure. &
This state is feasible but underspecified. It resembles the No LLM candidates
behavior: close to the inventory, but poor as an organizational rule. \\
Accepted MCMC edits &
The sampler replaces generic scope with \texttt{waiter}, \texttt{cook}, and
\texttt{duty\_manager}; coordinates the trigger with a 15-minute arrival window;
and swaps vague evidence for \texttt{movement\_log}, \texttt{checkin\_log}, and
\texttt{checkout\_log}. &
These edits increase intent alignment and procedural closure while preserving
record grounding. Unsupported variants such as a generic attendance-system
timestamp receive lower structural scores because they are not in the record
vocabulary. \\
Selected AOG and rendering &
The selected structure binds scope, trigger, norm body, evidence, and procedure
before the natural-language renderer writes the final clause. &
The final prose is therefore a realization of an inspectable structure. If the
rule later fails \hardok{}, the failure can be localized to selected payloads
rather than treated as an opaque text-generation error. \\
\bottomrule
\end{tabularx}
\caption{Worked AOG construction and MCMC search trace for a representative
service-shift goal. The example illustrates the intermediate commitments behind
the outputs in Table~\ref{tab:qual-main-app}.}
\label{tab:qual-aog-trace}
\end{table*}

Table~\ref{tab:qual-main-app} shows three outputs for the same shared
service-shift goal: coordinate arrivals and shift closure with auditable
evidence. Direct Prompting writes plausible workplace policy prose, but invents
an attendance system and timestamp evidence not present in the available record
layer. No LLM candidates stays inside the operational inventory and passes
\hardok{}, but the result is templatic and underspecified as policy language.
\ours{} writes the richest and most faithful clause, adding a 15-minute arrival
window and explicit \texttt{checkin\_log}/\texttt{checkout\_log} verification;
however, its support ratio falls below the binary gate because more selected
payloads are inferred rather than directly supported. This is the concrete form
of the quality-feasibility trade-off in Table~\ref{tab:main}.

\begin{table*}[!t]
\centering
\footnotesize
\setlength{\tabcolsep}{4pt}
\begin{tabularx}{\textwidth}{@{}p{0.18\textwidth}YYY@{}}
\toprule
\textbf{Claim} & \textbf{Representative excerpt} & \textbf{Diagnostic issue} &
\textbf{Score signal} \\
\midrule
Direct Prompting is fluent but unsupported &
``documented in the designated attendance system''; evidence from
``timestamps recorded in the attendance system'' &
The prose is coherent, but the support audit marks the attendance-system
timestamp and shift-commencement payloads as unavailable in the record layer. &
Avg \(=3.6\), \hardok{} = no \\
No LLM candidates is feasible but rigid &
``Covered members must perform \texttt{CheckIn} and \texttt{CheckOut}'' and
``record the actor role, scene, object, action, evidence channel, and observed
value'' &
The clause passes the feasibility gate by staying close to known actions and
logs, but reads like a fixed template and gives little substantive policy
guidance beyond slot enumeration. &
Avg \(=3.2\), \hardok{} = yes \\
\ours{} is richer but \hardok{} can be fragile &
``\texttt{CheckIn} within \texttt{trigger\_window\_minutes=15}'' using
\texttt{movement\_log}, \texttt{checkin\_log}, and \texttt{checkout\_log} &
The clause is more specific and faithful, but its real-data grounding ratio is
0.475: six payloads, including \texttt{shift}, \texttt{checkin\_record}, and
\texttt{checkout\_record}, lack direct resource support under the binary gate. &
Avg \(=4.6\), \hardok{} = no \\
\bottomrule
\end{tabularx}
\caption{Qualitative examples for the main method comparison on a shared
service-shift goal. The examples illustrate why the paper reports both rubric
quality and \hardok{} rather than collapsing them into a hidden score.}
\label{tab:qual-main-app}
\end{table*}

Table~\ref{tab:qual-slot-app} shows what the slot intervention changes. These
examples are useful because they make the heatmap in
Figure~\ref{fig:slot-sensitivity} concrete: corrupting scope changes who and
what the rule governs, corrupting the trigger changes when it activates, and
corrupting the norm body changes the required or forbidden conduct.

\begin{table*}[!t]
\centering
\footnotesize
\setlength{\tabcolsep}{4pt}
\begin{tabularx}{\textwidth}{@{}p{0.14\textwidth}YYY@{}}
\toprule
\textbf{Degraded slot} & \textbf{Correct slot content} &
\textbf{After degradation} & \textbf{Concrete error} \\
\midrule
\slot{Scope} &
The service-shift rule covers \texttt{waiter}, \texttt{cook}, and
\texttt{duty\_manager} in the \texttt{restaurant}, with
\texttt{shift}, \texttt{checkin\_record}, and
\texttt{checkout\_record}. &
The degraded rule adds \texttt{organization\_member} and
\texttt{patrol\_member}, and swaps in \texttt{learning\_material},
\texttt{play\_area}, and \texttt{shared\_resource}. &
The rule becomes about the wrong actors and objects, so the shift-attendance
goal is no longer the governing context. Avg drops from 3.2 to 2.0. \\
\slot{Trigger} &
The same service-shift rule activates in the \texttt{restaurant} when the
\texttt{CheckIn} \(\rightarrow\) \texttt{CheckOut} sequence is checked. &
The degraded rule activates in \texttt{public\_square} around
\texttt{Speak} and forbidden \texttt{Interact}. &
The activation condition moves from shift arrival/closure to unrelated public
interaction, so the rule fires at the wrong time. Avg drops from 3.2 to 2.0. \\
\slot{NormBody} &
An emergency-response rule requires \texttt{Patrol} and forbids
\texttt{Trade} under resource pressure. &
The degraded norm requires speaking to an owner before \texttt{Harvest} and
forbids false excuses before leaving a shift without checkout. &
The core duty is replaced by unrelated harvest and shift-excuse obligations,
so the rule no longer implements emergency response. Avg drops from 3.4 to
1.8. \\
\bottomrule
\end{tabularx}
\caption{Qualitative examples from the slot-sensitivity intervention. Each row
uses a paired direct-render baseline and a single-slot-degraded output.}
\label{tab:qual-slot-app}
\end{table*}

\raggedbottom
\section{Prompt Templates}
\label{app:prompts}

We include the prompt templates used by the baseline generators, auxiliary
generators, polishing stage, and evaluation rubrics. The frozen versions are
those used for reported scoring or polish stability checks; additional versions
are included to make template variation explicit.

This is the only appendix section that uses verbatim prompt blocks. Prompts are
typeset in shaded monospace blocks because line breaks, Markdown structure, and
placeholder names are part of the experimental material. We use a compact
monospace size so long templates do not dominate page layout. In contrast, result
summaries and audits are presented in prose and ordinary tables above rather
than copied verbatim into the paper.
Template placeholders denote schema fields and operational record fields in the
terminology of Section~\ref{sec:preliminaries}, not additional benchmark
concepts.

\promptinput{Direct Prompting}{b1_vanilla.v1.txt}

\promptinput{Structured-output Prompting}{b2_schema.v1.txt}

\promptinput{In-context Grounding}{b2_grounding.v1.txt}

\promptinput{Self-refine}{b2_refine.v1.txt}

\promptinput{Constrained-decoding Prompt}{b6_constrained.v1.txt}

\promptinput{Norm-synthesis Prompt}{b7_crsec.v1.txt}

\promptinput{LLM-augmented Candidate Generator}{llm_augmented_generator.v1.txt}

\promptinput{Reverse Parser}{reverse_parser.v1.txt}

\promptinput{Judge Prompt v1}{judge.v1.txt}

\promptinput{Independent Rubric v1}{independent_rubric.v1.txt}

\promptinput{Independent Rubric v2}{independent_rubric.v2.txt}

\promptinput{Frozen Independent Rubric v1}{independent_rubric.FROZEN.v1.txt}

\promptinput{Frozen Independent Rubric v2}{independent_rubric.FROZEN.v2.txt}

\promptinput{NL Polish Prompt v1}{nl_polish.v1.txt}

\promptinput{NL Polish Prompt v2}{nl_polish.v2.txt}

\promptinput{Frozen NL Polish Prompt v1}{nl_polish.FROZEN.v1.txt}

\promptinput{Frozen NL Polish Prompt v2}{nl_polish.FROZEN.v2.txt}

\promptinput{RealCharter-Bench Direct Prompt}{real_charter_b1.v1.txt}

\promptinput{RealCharter-Bench Self-refine Draft Prompt}{real_charter_b2r_draft.v1.txt}

\promptinput{RealCharter-Bench Self-refine Refine Prompt}{real_charter_b2r_refine.v1.txt}

\end{document}